\documentclass[sigconf]{acmart}

\usepackage{amsmath}
\usepackage{colortbl}
\usepackage{mathtools}
\usepackage{amsthm}
\usepackage{amsfonts,amscd,keyval}
\usepackage{bm}
\usepackage{multirow}
\usepackage{dsfont}
\usepackage{float}
\usepackage{xspace}
\usepackage{bbding}
\usepackage{graphics}
 \usepackage{enumitem}
 \usepackage{breqn}
\usepackage{listings}
\usepackage{xcolor}
\usepackage{algpseudocode}
\usepackage{algorithm}
\newcommand{\method}{\textsc{R-Mixup}\xspace}

\definecolor{codegreen}{rgb}{0,0.6,0}
\definecolor{codegray}{rgb}{0.5,0.5,0.5}
\definecolor{codepurple}{rgb}{0.58,0,0.82}
\definecolor{backcolour}{rgb}{0.95,0.95,0.92}

\lstdefinestyle{mystyle}{
  backgroundcolor=\color{backcolour}, commentstyle=\color{codegreen},
  keywordstyle=\color{magenta},
  numberstyle=\tiny\color{codegray},
  stringstyle=\color{codepurple},
  basicstyle=\ttfamily\footnotesize,
  breakatwhitespace=false,         
  breaklines=true,                 
  captionpos=b,                    
  keepspaces=true,                 
  numbers=left,                    
  numbersep=5pt,                  
  showspaces=false,                
  showstringspaces=false,
  showtabs=false,                  
  tabsize=2
}
\theoremstyle{plain}
\newtheorem{theorem}{Theorem}[section]
\newtheorem{proposition}[theorem]{Proposition}

\theoremstyle{definition}
\newtheorem{definition}[theorem]{Definition}

\theoremstyle{remark}
\newtheorem{remark}[theorem]{Remark}

\AtBeginDocument{%
  \providecommand\BibTeX{{%
    \normalfont B\kern-0.5em{\scshape i\kern-0.25em b}\kern-0.8em\TeX}}}

\copyrightyear{2023}
\acmYear{2023}
\setcopyright{acmlicensed}\acmConference[KDD '23]{Proceedings of the 29th
ACM SIGKDD Conference on Knowledge Discovery and Data Mining}{August
6--10, 2023}{Long Beach, CA, USA}
\acmBooktitle{Proceedings of the 29th ACM SIGKDD Conference on Knowledge
Discovery and Data Mining (KDD '23), August 6--10, 2023, Long Beach, CA,
USA}
\acmPrice{15.00}
\acmDOI{10.1145/3580305.3599483}
\acmISBN{979-8-4007-0103-0/23/08}

\begin{document}

\title{\method: Riemannian Mixup for Biological Networks}


\author{Xuan Kan}
\authornote{Both authors contributed equally to this research.}
\email{xuan.kan@emory.edu}
\affiliation{
\institution{Department of Computer Science, Emory University}
\city{Atlanta}
\state{GA}
\country{USA}
}

\author{Zimu Li}
\authornotemark[1]
\email{lizm@mail.sustech.edu.cn}
\affiliation{
\institution{Pritzker School of Molecular Engineering, University of Chicago}
\city{Chicago}
\state{IL}
\country{USA}
}

\author{Hejie Cui}
\email{hejie.cui@emory.edu}
\affiliation{
\institution{Department of Computer Science, Emory University}
\city{Atlanta}
\state{GA}
\country{USA}
}

\author{Yue Yu}
\email{yueyu@gatech.edu}
\affiliation{
\institution{School of Computational Science and Engineering, Georgia Institute of Technology}
\city{Atlanta}
\state{GA}
\country{USA}
}

\author{Ran Xu}
\email{ran.xu@emory.edu}
\affiliation{
\institution{Department of Computer Science, Emory University}
\city{Atlanta}
\state{GA}
\country{USA}
}

\author{Shaojun Yu}
\email{shaojun.yu@emory.edu}
\affiliation{
\institution{Department of Computer Science, Emory University}
\city{Atlanta}
\state{GA}
\country{USA}
}

\author{Zilong Zhang}
\email{201957020@uibe.edu.cn}
\affiliation{
\institution{School of Statistics, University of International Business and Economics}
\state{Beijing}
\country{China}
}

\author{Ying Guo}
\email{yguo2@emory.edu}
\affiliation{
\institution{Department of Biostatistics and Bioinformatics, Emory University}
\city{Atlanta}
\state{GA}
\country{USA}
}

\author{Carl Yang}
\email{j.carlyang@emory.edu}
\authornote{To whom correspondence should be addressed.}
\affiliation{
\institution{Department of Computer Science, Emory University}
\city{Atlanta}
\state{GA}
\country{USA}
}


\begin{abstract}
Biological networks are commonly used in biomedical and healthcare domains to effectively model the structure of complex biological systems with interactions linking biological entities. However, due to their characteristics of high dimensionality and low sample size, directly applying deep learning models on biological networks usually faces severe overfitting. In this work, we propose \method, a Mixup-based data augmentation technique that suits the symmetric positive definite (SPD) property of adjacency matrices from biological networks with optimized training efficiency. The interpolation process in \method leverages the log-Euclidean distance metrics from the Riemannian manifold, effectively addressing the \textit{swelling effect} and \textit{arbitrarily incorrect label} issues of vanilla Mixup. We demonstrate the effectiveness of \method with five real-world biological network datasets on both regression and classification tasks. Besides, we derive a commonly ignored necessary condition for identifying the SPD matrices of biological networks and empirically study its influence on the model performance. The code implementation can be found in Appendix \ref{appendix:code}. 
\end{abstract}






\maketitle

\section{Introduction}\label{sec:intro}

As a ubiquitous type of data in biomedical studies, biological networks are used to depict a complex system with a set of interactions between various biological entities. For example, in a brain network, the correlations extracted from functional Magnetic Resonance Imaging (fMRI) are modeled as interactions among human-divided brain regions~\cite{modellingfmri, simpson2013analyzing,wangfilter,lin2022deep,chen2022brainnet,kan2022fbnetgen, huang2022robust, huang2023robust}. Meanwhile, in a co-expression gene-protein network, interactions are built to discover disease genes and potential modules for clinical intervention~\cite{paci2021gene}. There are diverse ways to define the connections among entities in biological networks, such as interactions~\cite{yu2021sumgnn,li2023dsn}, reactions~\cite{craciun2006multiple}, and relations~\cite{kan2021zero,cui2022survey,xu2022counterfactual, braingb, cui2022interpretable}. One of the most widespread practices is calculating the covariance and correlation among entities to summarize and quantify interactions~\cite{BHAVSAR201892,simpson2013analyzing,yu2007importance, du2021graphgt, wang2023domain, Shiyucorrleted}.
Therefore, developing powerful computational methods to predict disease outcomes based on profiling datasets from such correlation matrices has attracted great interest from biologists~\cite{li2020braingnn,groupinn2019,Anirudh2019BootstrappingGC2019, yu2022learning,DBLP:conf/miccai/LiDZZVD19, lin2023magnetic, yang2023pretrain, ji2023prediction}. 

Deep learning methods have achieved state-of-the-art performance in various downstream applications~\cite{chu2021Twins,BrainNetCNN}, especially when the training sample size is large enough. However, biological network datasets often suffer from limited samples due to the complicated and expensive collection and annotation processes of scientific data~\cite{ABCD,localboost,xu2023weakly}. Another key property of biological networks is that the dimension of such networks is typically very high, i.e., $O(n^2)$ correlation edges among $n$ entities. Therefore, directly applying Deep Neural Networks (DNNs) to such biological network datasets can easily cause severe overfitting~\cite{arpit2017closer,xu2023neighborhood,yu2022cold, wei2022hiera, yangkdd}.

\begin{figure}
	\centering
	\includegraphics[width=1.0\linewidth]{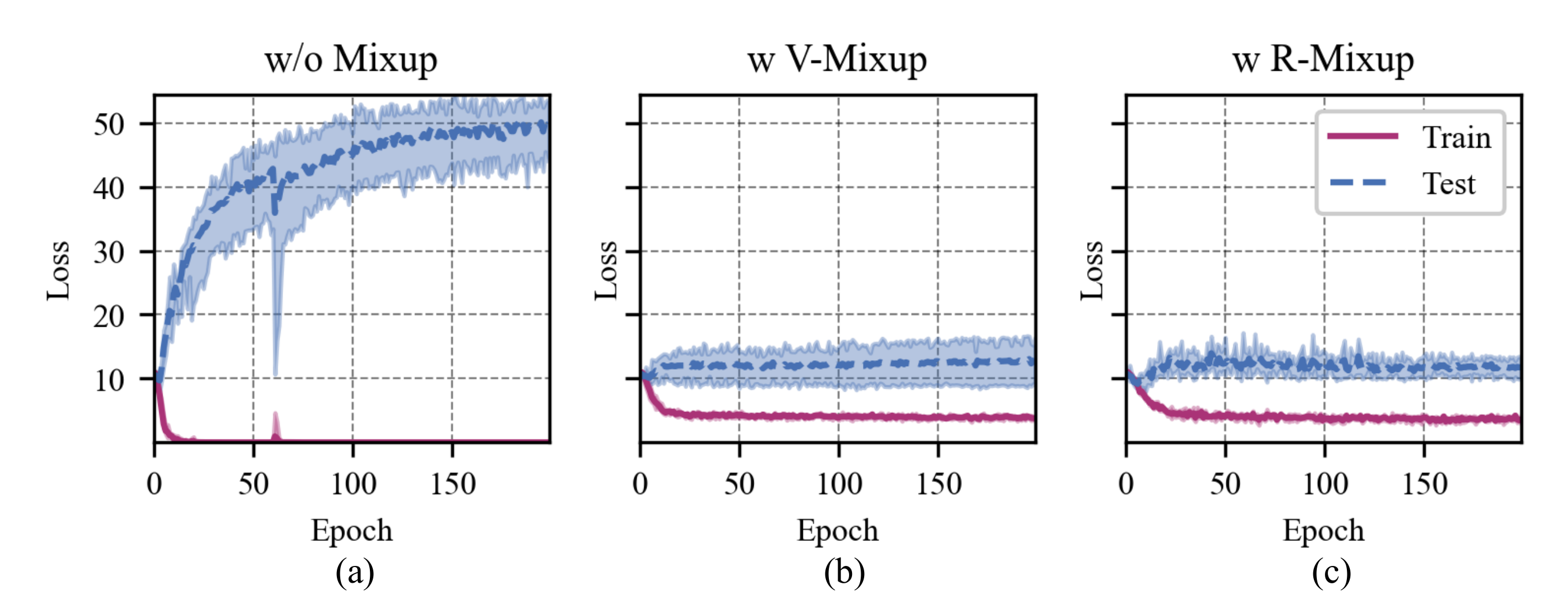}
	\caption{Train/Test performance of a Transformer on the biological network dataset PNC with  503 samples. Each sample is represented as a $120 \times 120$ adjacency matrix. V-Mixup is the vanilla Mixup and \method is our proposed method.}
 \vspace{-1.5ex}
	\label{fig:motivation1}
 \vspace{-3ex}
\end{figure}

Mixup is a widely used data augmentation technique that can improve the model performance by linearly interpolating new samples from pairs of existing instances~\cite{zhang2018mixup}. In the scenario of biological network analysis, since the node identities and their corresponding order are usually fixed across network samples within the same dataset~\cite{kan2022bnt}, the Mixup technique can be easily applied via linear interpolation. Empirically, Figure~\ref{fig:motivation1} (a) and (b) compare the processes of training a transformer model~\cite{kan2022bnt} without Mixup and with the vanilla Mixup (V-Mixup)~\cite{zhang2018mixup} technique on the brain network dataset from the PNC studies~\cite{pnc} to perform binary classification. 
In Figure~\ref{fig:motivation1} (a), the training loss without the Mixup technique diminishes quickly while the test loss continues to increase, which apparently indicates a severe overfitting problem. In contrast, in Figure~\ref{fig:motivation1} (b) with V-Mixup, the training process becomes more stable, and the model achieves higher performance with a lower test loss, even though the training loss is relatively high.

Although the vanilla Mixup can mitigate the overfitting issue for biological networks, there are two critical limitations in existing Mixup methods. The first noticeable issue is that the linear Mixup of correlation matrices in the Euclidean space would cause a \emph{swelling effect}, where the determinant of the interpolated matrix is larger than any of the original ones. The inflated determinant, which equals the product of eigenvalues, also indicates an increase in eigenvalues. This can be interpreted as exaggerated variances of the data points in the principal eigen-directions. As a result, an unphysical augmentation from original data is generated, which may change the characteristics, e.g., the correlations of different brain functional areas, of the original dataset and violate the intuition of linear interpolations that the determinant from a mixed sample should be intuitively \emph{between} the original pair of samples \cite{Chefdhotel2004,Feddern2006,Pennec2009,Dryden2010}.
On the other hand, the vanilla Mixup cannot properly handle regression tasks due to \emph{arbitrarily incorrect label}~\cite{yao2022cmix}, which means that linearly interpolating a pair of examples and their corresponding labels cannot ensure that the synthetic sample is paired with the correct label. Although several existing works like RegMix~\cite{regmix} and C-Mixup~\cite{yao2022cmix} have attempted to avoid this issue by restricting the mixing process only to samples with a similar label, their practice leads to less various sample generation and weakens the ability of Mixup towards improving the robustness and generalization of deep neural network models.

\begin{figure}
    \centering
    \includegraphics[width=0.75\linewidth]{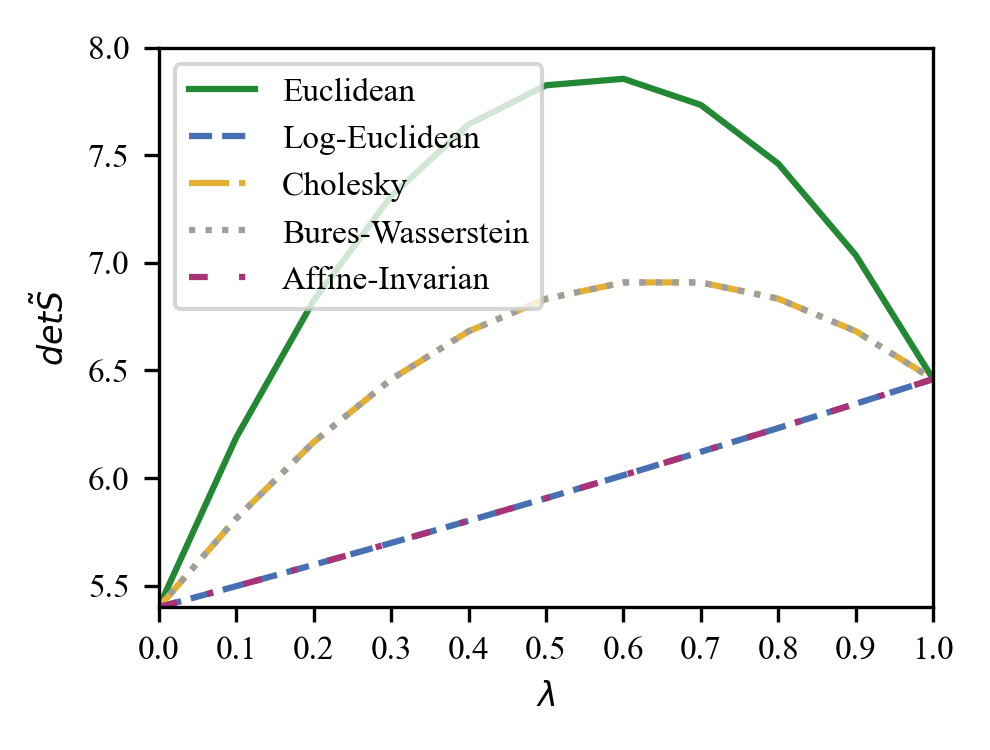}
    \caption{The \textit{swelling effect} of Mixing up with different metrics. $\tilde{S}$ is the augmented sample mixed by samples $S_i$ and $S_j$, where $\det S_i = 5.40$ and $\det S_j = 6.46$. Ideally, the determinant of the mixed sample $\tilde{S}$ should be between $\det S_i$ and $\det S_j$. The results indicate that mixing samples with Euclidean (widely used in existing Mixup methods), Cholesky, and Bures-Wasserstein metrics leads to unphysical inflations. }
    \label{fig:motivation2}
     \vspace{-2ex}
\end{figure}

Recently, investigating covariance and correlation matrices in the view of symmetric positive definite matrices (SPD) with Riemannian manifold has demonstrated impressive advantages in biological domains~\cite{barachant2010riemannian,you2022geometric, pan2022matt}, which helps to improve the model performance and capture informative sample features. Inspired by these studies, we pinpoint a promising direction to mitigate these two identified issues when adapting the Mixup technique for biological networks from the perspective of SPD analysis. However, existing works that leverage the Riemannian manifold for SPD analysis of biological networks often directly treat covariance and correlation matrices as SPD matrices without rigorous verification. We clarify that covariance and correlation matrices are not equal to SPD matrices: a necessary condition for the covariance and correlation matrices generated from a sample $X \in \mathbb{R}^{ n \times t}$ to be positive definite is that \textit{the sequence length $t$ is no less than the sample variable number $n$}. We provide theoretical proof for this condition in Appendix~\ref{sec:SPD}. The collection of positive definite matrices mathematically forms a unique geometric structure called the \textit{Riemannian manifold}, which generalizes curves and surfaces to higher dimensional objects \cite{Lee2018,Gallier2020,Sakai1996}. From a mathematical perspective, augmenting samples along geodesics on the manifold of SPDs with the log-Euclidean metric effectively (a) preserves the intrinsic geometric structure of the original data and eases the \emph{arbitrarily incorrect label} and (b) eliminates the \emph{swelling effect} as is shown in Figure \ref{fig:motivation2}. The advantages are further proved theoretically in Section \ref{sec:method}.

Based on this insight, we propose \method, a Mixup-based data augmentation approach for SPD matrices of biological networks, which augments samples based on Riemannian geodesics (i.e., Eq.\eqref{eq:geodesic}) instead of straight lines ( i.e., Eq.\eqref{eq:line}). 
We theoretically analyze the advantages of \method by incorporating tools from differential geometry, probability and information theory. Besides, a simple and efficient preprocess optimization is proposed to reduce the actual training time of \method considering the costly eigenvalue decomposition operation.
Sufficient experiments on five datasets spanning both regression and classification tasks demonstrate the superior performance and generalization ability of \method. 
For regression tasks, \method can achieve the best performance based on the same random sampling strategy as vanilla Mixup, demonstrating its ability to overcome the \emph{arbitrarily incorrect label} issue by adequately leveraging the intrinsic geometric structure of SPD. This advantage is also proved by a case study in Appendix \ref{appendix:case_study}. 
Furthermore, we observe that the performance gain of \method over existing methods is especially prominent when the annotated samples are extremely scarce, verifying its practical advantage under the low-resource settings. 

We summarize the contributions of this work as three folds:
\begin{itemize}[nosep,leftmargin=*]
\item We propose \method, a data augmentation method for SPD matrices in biological networks, which leverages the intrinsic geometric structure of the dataset and resolves the \textit{swelling effect} and \textit{arbitrarily incorrect label} issues. Different Riemannian metrics on manifold are compared, and the effectiveness of \method is theoretically proved from the perspective of statistics. We also proposed a pre-computing optimization step to reduce the burden from eigenvalue decomposition.

\item Thorough empirical studies are conducted on five real-world biological network datasets, demonstrating the superior performance of \method on both regression and classification tasks. Experiments on low-resource settings further stress its practical benefits for biological applications often with limited annotation.

\item We emphasize a commonly ignored necessary condition for viewing covariance and correlation matrices as SPD matrices. We believe the clarification of this pre-requirement for applying SPD analysis can enhance the rigor of future studies.
\end{itemize}
\section{Related Work}

\subsection{Mixup for Data Augmentation}
Mixup is a simple but effective principle to construct new training samples for image data by linear interpolating input pairs and forcing the DNNs to behave linearly in-between training examples~\cite{zhang2018mixup}. 
Many follow-up works extend Mixup from different perspectives. 
For example, \cite{verma2019manifoldmixup,venkataramanan2022alignmixup} interpolate training data in the feature space, 
\cite{guo2019mixup,chou2020remix} learn the mixing ratio for Mixup to alleviate the under-confidence issue for predictions. 
Besides, 
\cite{dabouei2021supermix,hwang2021mixrl,zhang2022m,yao2022cmix} strategically select the sample pairs for Mixup to prevent low-quality mixing examples and produce more reasonable augmented data.
To further improve the quality of the augmented data, \cite{yun2019cutmix,kim2020puzzle,hendrycks2022pixmix} create mixed examples by only interpolating a specific region (often most salient ones) of examples. {Mixup has also been extended to other data modalities such as text~\cite{chen2020mixtext,zhang2020seqmix} and audio~\cite{meng2021mixspeech}. There are several attempts to study Mixup on  non-Euclidean data, graphs, like NodeMixup~\cite{wang2021mixup}, GraphMixup~\cite{wu2021graphmixup} and G-Mixup~\cite{han2022gmixup}. However, less attention has been paid to adapting Mixup for graphs from a manifold perspective,  which is the focus of this study.} 

\subsection{Geometric Deep Learning}

Geometric deep learning aims to adapt commonly used deep network architectures from euclidean data to non-euclidean data, such as graphs and manifolds, with a broad spectrum of applications from the domains of radar data processing~\cite{brooks2019riemannian}, graph analysis~\cite{monti2017geometric,aykent2022gbpnet}, image and video processing~\cite{huang2017spdnet, ManifoldNet, monti2017geometric,he2020curvanet}, and Brain-Computer Interfaces~\cite{suh2021riemannian,pan2022matt}. 
For example, SPDNet~\cite{huang2017spdnet} builds a Riemannian neural network architecture with special convolution-like layers, rectified linear units (ReLU)-like layers, and modified backpropagation operations for the non-linear learning of SPD matrices. 
ManifoldNet~\cite{ManifoldNet} defines the analog of convolution operations for manifold-valued data. 
MoNet~\cite{monti2017geometric} generalizes CNN architectures to the non-Euclidean domain with pseudo-coordinates and weight functions. \cite{brooks2019riemannian} designs a Riemannian batch normalization for SPD matrices by leveraging geometric operations on the Riemannian manifold. MAtt~\cite{pan2022matt} proposes the manifold attention mechanism to represent spatiotemporal representations of EEG data. Though widely recognized as being effective for images, tabular and graph data, to the best of our knowledge, data augmentation methods in geometric deep learning have rarely been explored.
\section{\method} \label{sec:method}

In this section, we first provide some preliminary facts, including a necessary condition for treating covariance and correlation matrices as SPD matrices. Next, we elaborate on the detailed process of applying \method for data augmentation, compare possible mathematical metrics designs, and finally provide the theoretical analysis of the advantages of using \method. 

\subsection{Notations and Preliminary Results}
Given $n$ variables of biological entities, we extract a $t$ length sequence for each variable and compose the input sequences $X \in \mathbb{R}^{ n \times t}$. The correlation matrix or biological network $S = \text{Cor}(X) \in \mathbb{R}^{n \times n}$ is obtained by taking the pairwise correlation among each pair of the biological variables. The value $y$ is the network-level prediction label for the prediction task.

\begin{definition}\label{def:SPD}
	A symmetric $n \times n$ matrix $S$ is \emph{positive semi-definite} if for any vector $u \in \mathbb{R}^n$, $u^T S u \geq 0$. Equivalently, this means that the eigenvalues of $S$ are all nonnegative. If the inequality holds strictly, $S$ is said to be \emph{positive definite}, or \emph{symmetric positive definite}, or SPD for short.
\end{definition}

Let $\text{Sym}(n)$ be the collection of all positive semi-definite matrices, and $\text{Sym}^+(n)$ denotes the collection of all SPDs. The collection $\text{Sym}(n)$ can be seen as an $\frac{1}{2}n(n-1)$-dimensional Euclidean space, but $\text{Sym}^+(n) \subset \mathbb{R}^{n \times n}$ admits a more general structure call \emph{manifold} in differential geometry which resembles the Euclidean space in its local regions. To set up the modeling on the manifold $\text{Sym}^+(n)$, the covariance matrix $\text{Cor}(X)$ for the input $X$ should be positive definite. However, it is worth mentioning that previous studies that use the Riemannian manifold for analyzing biological networks often treat covariance and correlation matrices as SPD without proper validation. Towards this common negligence, we bring out the following basic fact:
\begin{proposition}\label{prop:MCondition}
Covariance and correlation matrices are positive semi-definite. A necessary condition for them to be positive definite is that the sample length is no less than the variable number, i.e., $t \geq n$.
\end{proposition}
This proposition indicates that covariance and correlation matrices only have the opportunity to be positive definite when $t \geq n$. The detailed proof can be found in Appendix \ref{sec:SPD}. This is the case for the datasets involved in this study, where most of the correlation matrices are SPD. The few exceptions would have very few zero eigenvalues, which we manually set as $10^{-6}$ to eliminate their influence. More discussions on adjusting correlation matrices to be SPDs can be found in \cite{Congedo2015, Huang2017}. 

\subsection{\method Deduction}
In this section, we explain on the detailed process of \method for SPD matrice augmentation. 
Let $S_i,S_j$ represent two different correlation matrices constructed based on $X_i, X_j$. In the vanilla Mixup~\cite{zhang2018mixup}, the augmented samples $(\tilde{S}, \tilde{y})$ are created through the straight line connecting $S_i, S_j$ and $y_i, y_j$,
\begin{align}\label{eq:line}
	\begin{split}
		\tilde{S} &= (1-\lambda) S_i + \lambda S_j, \\
		\tilde{y} &= (1 - \lambda) y_i + \lambda y_j,
	\end{split}
\end{align} 
where $\lambda \sim \operatorname{Beta}(\alpha, \alpha)$, $\operatorname{Beta}$ is the Beta distribution, given $\alpha \in(0, \infty)$.

To facilitate the illustration of \method in geometry notions, we briefly introduce the main concepts here while more detailed explanations can be found in \cite{Lee2018,Gallier2020}. To define \method, we replace Eq.\eqref{eq:line} by a certain Riemannian geodesics. \emph{Geodesics} are the generalization of \emph{straight lines} in the Euclidean space, which is intuitively the shortest path between two given points on Riemannian manifolds. Riemannian manifolds $(M,g)$ are manifolds $M$ equipped with \emph{Riemannian metrics} $g$ which measure distances between points in the manifold and induces geodesic equations \cite{Lee2018,Gallier2020}. It is generally hard to solve geodesics equations in the simple analytical form as straight lines, however, for $\text{Sym}^+(n)$, there are lots of well-defined choices of Riemannian metrics with known geodesics~\cite{Pennec2009,Jayasumana2014,Dryden2010,Gallier2020}, and we employ the \emph{log-Euclidean metric} with the following geodesic: 
\begin{align}\label{eq:geodesic}
	\tilde{S} = \exp \left((1-\lambda) \log S_i + \lambda \log S_j \right),
\end{align} 
where $\exp, \log$ are matrix exponential and logarithm.
Figure \ref{fig:visual} sketches the geodesic as the purple dotted curve and a rigorous deduction of Eq.\eqref{eq:geodesic} can be found in \cite{Gallier2020}. Implementation of the matrix exponential for positive definite matrix $S$ is straightforward: by basic linear algebra, 
\begin{align}
	S = O \text{diag}(\mu_1,...,\mu_n) O^T,
\end{align}
where $O$ is an orthogonal matrix with $\mu_i$ being eigenvalues of $S$. Then by definition,
\begin{align}\label{eq:log_exp}
	\begin{split}
		\exp S &= O \text{diag}(\exp \mu_1,...,\exp \mu_n) O^T, \\
		\log S &= O \text{diag}(\log \mu_1,...,\log \mu_n) O^T.
	\end{split}
\end{align}

\begin{figure}
	\centering
	\includegraphics[width=0.95\linewidth]{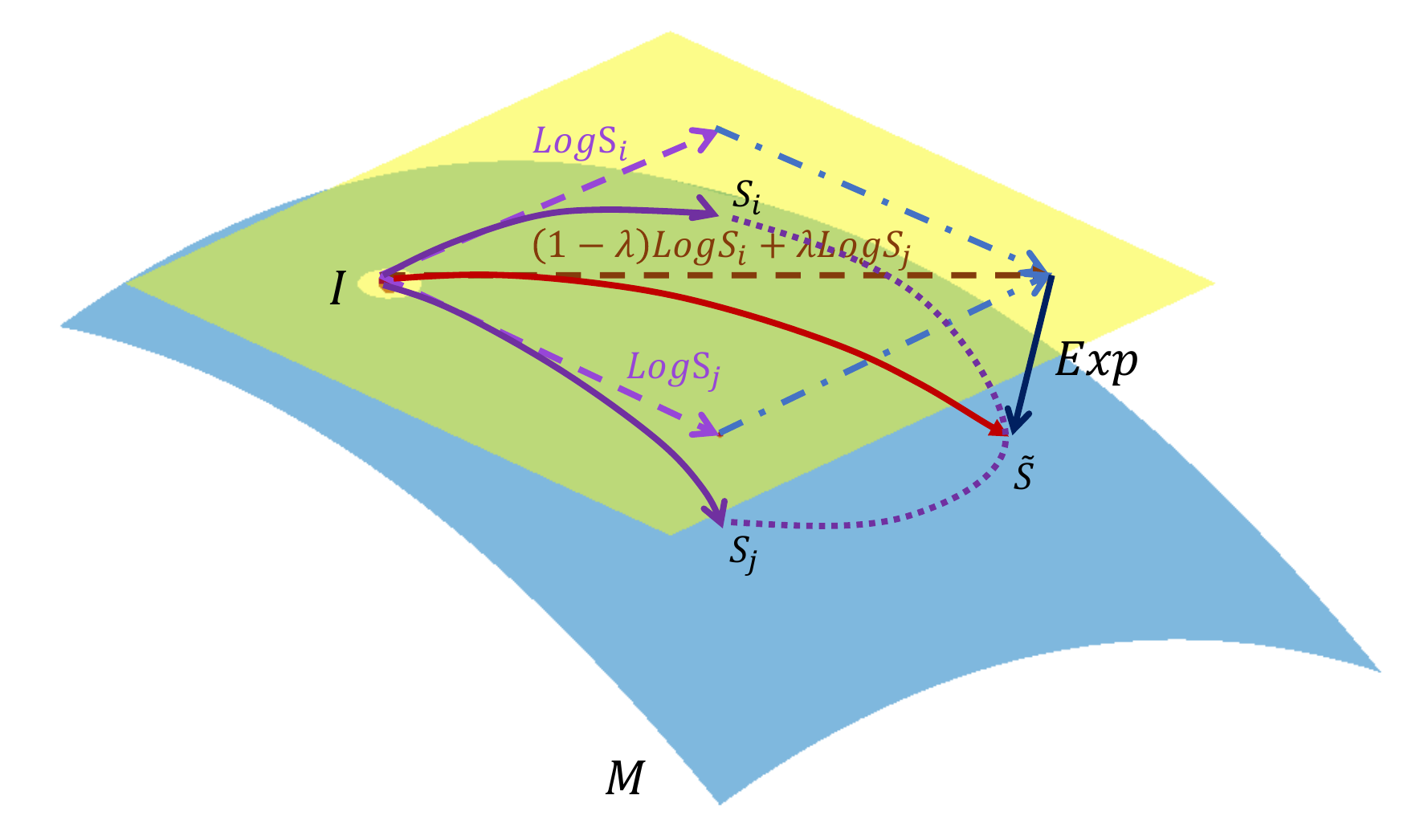}
	\caption{The process of \method generating sample $\tilde{S}$, where the blue surface $M$ represents the \textit{Riemannian manifold} and the yellow plane is the tangent plane of $M$ at the origin $I$. $S_i, S_j$ are the original samples in $M$, and $\log S_i, \log S_j$ are tangent vectors. \method creates the augmented sample $\tilde{S}$ by combining the initial tangent vectors of both trajectories connecting $I$ with $S_i, S_j$, i.e., $(1-\lambda) \log S_i + \lambda \log S_j$, and push it back to the \textit{Riemannian manifold} $M$ via exponential map.} 
	\label{fig:visual}
\end{figure}

\subsection{Comparison with Other Metrics}\label{sec:Compare}
There are various choices of Riemannian metrics and hence different geodesics on $\text{Sym}^+(n)$ \cite{Pennec2009,Dryden2010,Jayasumana2014,Bhatia2019a,Gallier2020}, such as the Cholesky metric defined by Cholesky decompositions $L_i$ of positive definite matrices $S_i$, the well-known Affine-invariant metrics on $\text{Sym}^+(n)$~\cite{Thanwerdas2023}, and the Bures-Wasserstein studied in statistics and information theory \cite{Bhatia2019a,Bhatia2019b}. We compare the most popular ones with the proposed log-Euclidean for mixing up biological networks on prediction tasks. The comparisons are summarized in Table \ref{tab:Metrics}. To be specific, different geodesics are analyzed from two perspectives: $(a)$ whether it causes the swelling effect, $(b)$ whether it is numerically stable on our dataset. 

\noindent\textbf{Swelling Effect.} The detailed definition and rigorous proof of the swelling effect can be found in Section \ref{sec:Theory} and Appendix \ref{sec:swell}. As exemplified by the motivation in Figure \ref{fig:motivation2}, Euclidean, Cholesky, and Bures-Wasserstein metrics evidently suffer from the swelling effect.

\noindent\textbf{Numerical Stability.} Augmenting matrices from the geodesic with the Affine-invariant metric requires the computation of $S_i^{-1/2}$ and hence the calculation of the inverse square root of its eigenvalues as we define matrix exponential and logarithm in Eq.\eqref{eq:log_exp}. For SPDs with small eigenvalues $\mu$, such computations may not be numerically stable since $\mu^{-1/2} \to \infty$. Furthermore, with awareness to the following limit relation:
\begin{align}
	\lim_{\mu \to 0} \frac{\log \mu}{\mu^{-1/2}} = 0,
\end{align}
which indicates that $\log \mu \ll \mu^{-1/2}$ for small $\mu$, we know that computing matrix logarithm when using log-Euclidean metric should be more stable. Similarly, for Bures-Wasserstein geodesics, to compute $(S_i S_j)^{1/2}$, we notice the following fact:
\begin{align}
	S_i S_j & = S_i^{1/2} \Big( S_i^{-1/2} (S_i S_j) S_i^{1/2} \Big) S_i^{-1/2} = S_i^{1/2} \Big( S_i^{1/2} S_j S_i^{1/2} \Big) S_i^{-1/2}. 
\end{align}
Thus,
\begin{align}
	(S_i S_j)^{1/2} = S_i^{1/2} \Big( S_i^{1/2} S_j S_i^{1/2} \Big)^{1/2} S_i^{-1/2}, 
\end{align}
where the undesirable $\mu^{-1/2}$ appears again in the calculation. 

Considering these two points, we stick with log-Euclidean metric. Experimental results in Section \ref{sec:overall} further showcase the effectiveness of this choice.

\begin{table*}[htbp]
	\centering
	\small
	\caption{Comparison of Different Metrics Choices}
	\label{tab:Metrics}
	\resizebox{0.75\linewidth}{!}{
		\begin{tabular}{cccc}
			\toprule
			\bf Metric &\bf Geodesics& \bf Swelling Effect &\bf Numerical Stability  \\
			\midrule
			Euclidean & $(1-\lambda)S_i + \lambda S_j$ & Yes & Stable \\
			Cholesky & $((1-\lambda)L_i + \lambda L_j) ((1-\lambda)L_i + \lambda L_j)^T$ & Yes & Stable \\
			Bures-Wasserstein & $ (1-\lambda)^2 S_i + \lambda^2 S_j + \lambda(1-\lambda)\big( (S_i S_j)^{1/2} + (S_j S_i)^{1/2} \big)$ & Yes & Unstable \\
			Affine-invariant & $S_i^{1/2} (S_i^{-1/2} S_j S_i^{-1/2})^\lambda S_i^{1/2}$ & No & Unstable \\
			\rowcolor{gray!10}\textbf{Log-Euclidean} & $\exp ( (1-\lambda)\log S_i + \lambda \log S_j)$ & \textbf{No} & \textbf{Stable} \\
			\bottomrule
		\end{tabular}
	}
\end{table*}

\subsection{\method Theoretical Justification}\label{sec:Theory}

Using geodesics when conducting data augmentation demonstrates unique advantages over straight lines. The first advantage is that \method will not cause the \textit{swelling effect} which exaggerates the determinant and certain eigenvectors of the samples as discussed in Section \ref{sec:intro} and \ref{sec:Compare}. Mathematically, suppose $\det S_i \leq \det S_j$, then the determinant of $\tilde{S}$ defined by Eq.\eqref{eq:geodesic} satisfies: 
\begin{align}
	\det S_i \leq \det \tilde{S} \leq \det S_j.
\end{align}
Detailed proof can be found in Appendix \ref{sec:swell}.

The second advantage is that, by leveraging the manifold structure, we can fit better estimators compared with linear interpolation in the Euclidean space. To be precise, as illustrated after Proposition \ref{prop:MCondition}, our samples are distributed over $\text{Sym}^+(n)$ rather than the whole ambient Euclidean space $\mathbb{R}^{n \times n}$, which is accepted as a \emph{prior knowledge} in the sense of Bayesian modeling fitting. Then the purpose of implementing \method becomes clear: we augment the nontrivial geometric information for the learning architectures later used in our experiments as an analogy to transforming images to enhance the translation and rotational invariance before a training of image identification \cite{Chen2020}. We theoretically justify this point from the perspective of both statistics on Riemannian manifolds \cite{Peter2007,Kim2014,Jayasumana2014,Jayasumana2015,Dodero2015} and information theory \cite{Carlen2009,Nielsen2010,Roberts2022}.

Specifically, we treat the data augmentation process as a regression conducted on the manifold $\text{Sym}^+(n)$ which is explicitly constrained by its geometric structure and based on the distribution of the dataset as the prior knowledge. Given any $\tilde{S}$, let $\tilde{m}(\tilde{S})$ denote the \emph{estimator/prediction function} of the regression whose analytical form depends on the concrete regression methods. We take geodesic regression and kernel regression \cite{Peter2007,Fletcher2011,ShaweTaylor2004} on $\text{Sym}^+(n)$ to address the problem. Roughly speaking, geodesic regression generalizes multi-linear regression on Euclidean space to manifold with the Euclidean distance being replaced by Riemannian metric. Kernel regression embeds data into higher dimensional feature space with kernel functions $K$ to grasp more non-linear relationship of the dataset. Since the exact distribution of augmented data is unknown, we follow the common practice \cite{Hardle1990,Cheng1990,ShaweTaylor2004} and apply \emph{Gauss kernel} 
\begin{align}\label{eq:EGauss}
	K_E(S_i, \tilde{S}) = \frac{1}{(2\pi \sigma^2)^{\frac{n}{2}}} \exp\left(- \frac{1}{2\sigma^2} \Vert S_i - \hat{S} \Vert^2\right),
\end{align}
which possess the \emph{universal property} to approximate any continuous bounded function in principle. However, the Gauss kernel $K_E$ is defined on the Euclidean space which \emph{unreasonably} implies non-zero density of samples outside $\text{Sym}^+(n)$ contradicting the prior knowledge. To remedy the problem, we introduce a method from the \emph{heat kernel theory} in differential geometry \cite{Berline2004,Gallier2020} to generalize $K_E$ to
\begin{align}\label{eq:HeatKernel}
	K_R(S_i, \hat{S}) = \frac{1}{(2\pi \sigma^2)^{\frac{n(n-1)}{4}}} \exp\left(- \frac{1}{2\sigma^2} d(S_i, \hat{S})^2\right),
\end{align}
with 
\begin{align}
	d(S_i,S_j) = \Vert \log S_i - \log S_j \Vert
\end{align}
being the \emph{Riemannian distances function} on $\text{Sym}^+(n)$. Then we prove in details in the Appendix \ref{sec:kernel}.

\begin{theorem}\label{Thm:comparison} 
	For $\text{Sym}^+(n)$ with log-Euclidean metric, comparing \method with estimators $\tilde{m}$ obtained by regressions with respect to the manifold structure, the square loss for augmented data $\tilde{S}$ from Riemannian geodesics Eq.\eqref{eq:geodesic} is no more than those $\tilde{S}'$ from straight lines Eq.\eqref{eq:line}:
	\begin{align}\label{eq:MSE}
		\sum (\tilde{m}(\tilde{S}) - \tilde{y})^2 \leq (\tilde{m}(\tilde{S}') - \tilde{y})^2.
	\end{align} 
\end{theorem}

A less empirical loss from regression on manifold is recognized as an evidence that \method captures some geometric features of $\text{Sym}^+(n)$, thereby providing the learning algorithm an opportunity to learn this feature. Finally, the proposed \method is formally defined as 
\begin{align}
	\begin{split}
		\tilde{S} &= \exp \left((1-\lambda) \log S_i + \lambda \log S_j\right), \\
		\tilde{y} &= (1 - \lambda) y_i + \lambda y_j,
	\end{split}
	\label{equ: r-mixup}
\end{align}
where $\lambda \sim \operatorname{Beta}(\alpha, \alpha)$, for $\alpha \in(0, \infty)$. 

\subsection{Time Complexity and Optimization}\label{sec:method_rt}

One potential concern of the proposed \method lies in its time-consuming operations of the eigenvalue decomposition and matrix multiplication (with the time complexity of $\mathcal{O}(n^3)$), which dominate the overall running time of \method. 
In practice, we find that most common modern deep learning frameworks such as PyTorch~\cite{NEURIPS2019_9015} have been optimized for accelerating matrix multiplication. Thus, the main extra time consumption of the \method is the $\exp$ and $\log$ operations of the three eigenvalue decompositions. 
We propose a sample strategy to optimize the running time of \method by precomputing the eigenvalue decomposition and saving the orthogonal matrix $O$ and eigenvalues $\{\mu_1,...,\mu_n\}$ of each sample. This precomputing process can reduce the three computations of eigenvalue decomposition to once for each sample. Formally,
\begin{align}
	\begin{split}
		\tilde{S}  = \exp \Big( &(1- \lambda)O_i \text{diag}(\log\mu_1,...,\log\mu_n) O_i^T \\
		+ &\lambda O_j \text{diag}(\log\nu_1,...,\log\nu_n) O_j^T \Big),
	\end{split}
\end{align}
where $O_i \text{diag}(\mu_1,...,\mu_n) O_i^T$ and $S_j = O_j \text{diag}(\nu_1,...,\nu_n) O_j^T$ are the eigenvalue decompositions of $S_i$ and $S_j$, respectively. The efficiency of this optimization is further discussed in Section \ref{sec:rq3}.
\begin{table*}[htbp]
\centering
\small
\caption{Dataset Summary.}
\vspace{-2ex}
\label{tab:dataset}
\resizebox{0.78\linewidth}{!}{
\begin{tabular}{ccccccc}
\toprule
\bf Dataset & \bf Sample Size&\bf Variance Number ($n$) &\bf Sequence Length ($t$) &\bf Task & \bf Class Number  \\
\midrule
ABCD-BioGender &7901 & 360 & Variable Length & Classification & 2 \\
ABCD-Cog &7749 & 360 & Variable Length  & Regression & -\\
PNC & 503 & 120 & 120 & Classification & 2  \\
ABIDE & 1009 & 100 & 100  & Classification & 2 \\
TCGA-Cancer & 240 & 50 & 50 & Classification & 24 \\
\bottomrule
\end{tabular}
}
\end{table*}

\section{Experiments} 
We evaluate the performance of \method comprehensively on real-world biological network datasets with five tasks spanning classification and regression. The dataset statistics are summarized in Table \ref{tab:dataset}. The empirical studies aim to answer the following three research questions: 
\begin{itemize}[nosep,leftmargin=*]
\item \textbf{RQ1}: How does \method perform compared with existing data augmentation strategies on biological networks with various sample sizes on different downstream tasks?
\item \textbf{RQ2}: How does the sequence length of each sample affect the characteristics of correlation matrices and consequently the choice of augmentation strategies?
\item \textbf{RQ3}: Is \method efficient in the training process and robust to hyperparameter changes?
\end{itemize}

\subsection{Experimental Setup}
\subsubsection{Datasets and Tasks}\hfill
 
\noindent\textbf{Adolescent Brain Cognitive Development Study (ABCD)}. The dataset used in this study is one of the largest publicly available fMRI datasets, with access restricted by a strict data requesting process~\cite{ABCD}. From this dataset, we define two tasks: \emph{BioGender Prediction} and \emph{Cognition Summary Score Prediction}. The data used in the experiments are fully anonymized brain networks based on the HCP 360 ROI atlas \cite{GLASSER2013105} with only biological sex labels or cognition summary scores. BioGender Prediction is a binary classification problem, which includes 7901 subjects after the quality control process, with 3961 (50.1\%) females among them. Cognition Summary Score Prediction is a regression task whose label is Cognition Total Composite Score containing seven computer-based instruments assessing five cognitive sub-domains: Language, Executive Function, Episodic Memory, Processing Speed, and Working Memory, ranging from 44.0 to 117.0.

\noindent\textbf{Autism Brain Imaging Data Exchange (ABIDE)}. The dataset includes anonymous resting-state functional magnetic resonance imaging (rs-fMRI) data from 17 international sites~\cite{abide}. It includes brain networks from 1009 subjects, with a majority of 516 (51.14\%) being patients diagnosed with Autism Spectrum Disorder (ASD). The task is to perform the binary classification for ASD diagnosis. The region definition is based on Craddock 200 atlas \cite{craddock2012whole}. Given the blood-oxygen-level-dependent (BOLD) signal length of the samples in this dataset is 100, which reflects whether neurons are active or reactive, we randomly select 100 nodes to satisfy the necessary condition discussed in Proposition~\ref{prop:MCondition} for SPD matrices. 

\noindent\textbf{Philadelphia Neuroimaging Cohort (PNC)}. The dataset is a collaborative project from the Brain Behavior Laboratory at the University of Pennsylvania and the Children's Hospital of Philadelphia. It includes a population-based sample of individuals aged 8–21 years \citep{pnc}. After the quality control, 503 subjects were included in our analysis. Among these subjects, 289 (57.46\%) are female. In the resulting data, each sample contains 264 nodes with time series data collected through 120 timesteps. Hence, we randomly select 120 nodes to satisfy the necessary condition mentioned in Proposition~\ref{prop:MCondition} for treating generated correlation matrices as SPD. BioGender Prediction is used as the downstream task.

\noindent\textbf{TCGA-Cancer Transcriptome}. 
The Cancer Genome Atlas (TCGA) dataset is a large-scale collection of multi-omics data from over 20,000 primary cancer and matched normal bio-samples spanning 33 cancer types. In this study, we select non-redundant cancer subjects with gene expression data and valid clinical information. The gene expression data is normalized, and the top 50 highly variable genes (HVG) are selected as the nodes for network construction. The subjects are then assigned to different samples based on their cancer subtype. The final dataset consists of 459 subjects from 66 cancer subtypes. We extract 240 correlation matrices from these subjects with 24 cancer types, each type includes ten samples, and each sample contains 50 nodes. The downstream task of this study is to predict cancer subtypes based on the HVG expression network.

\subsubsection{Metrics.} 
For binary classification tasks on datasets ABCD-BioGender, PNC, and ABIDE, we adopt AUROC and accuracy for a fair performance comparison. The classification threshold is set as 0.5. For the regression task on ABCD-Cog, the mean square error (MSE) is used to reflect model performance. For the multiple class classification task on TCGA-Cancer, since it contains 24 classes and each class has a balanced sample size, we take the macro Precision and macro Recall so that all classes are treated equally to reflect the overall performance. 
All the reported results are based on the average of five runs using different random seeds.

\begin{table*}[htbp]
\centering
\small
\caption{Overall performance comparison based on the Transformer backbone. The best results are in bold, and the second best results are \underline{underlined}. The $\uparrow$ indicates a higher metric value is better and $\downarrow$ indicates a lower one is better. }
\vspace{-1ex}
\label{tab:overall}
\resizebox{1.0\linewidth}{!}{
\begin{tabular}{ccccc cccc ccccc}
\toprule
\multirow{2.5}{*}{Method} &\multicolumn{2}{c}{\bf ABCD-BioGender}& & \multicolumn{1}{c}{\bf ABCD-Cog} & & \multicolumn{2}{c}{\bf PNC}& & \multicolumn{2}{c}{\bf ABIDE} & & \multicolumn{2}{c}{\bf TCGA-Cancer}\\
\cmidrule(lr){2-3} \cmidrule(lr){5-5} \cmidrule(lr){7-8} \cmidrule(lr){10-11} \cmidrule(lr){13-14}
& {AUROC$\uparrow$} & {Accuracy$\uparrow$}& { } & {MSE$\downarrow$}& { }  & {AUROC$\uparrow$}& {Accuracy$\uparrow$}& { } & {AUROC}$\uparrow$& {Accuracy$\uparrow$} & { }& {Precision$\uparrow$}& {Recall$\uparrow$}\\
\midrule
w/o Mixup &95.28±0.32 & 87.68±1.31 & { } & 60.21±1.53 & { } &  74.85±4.93 & 66.57±6.29 & { } & 73.32±4.11 & 66.00±3.66 & { } &35.33±11.52 & 45.00±10.79\\
\midrule
V-Mixup & 95.85±0.63 & 87.86±1.45 & { }& 60.43±2.67 & { }& 76.02±2.54 & 65.88±7.89  & { } & \textbf{75.03±5.04} &  {66.80±5.40}& { } & 69.58±9.39 &  \underline{77.50±6.97} \\
D-Mixup & 94.55±2.84 & 87.17±3.45 & { }& 60.96±1.82 & { }&  \underline{76.15±4.58} &  {68.82±6.29} & { } & 72.92±4.93 & \underline{67.40±5.64} & { } &\underline{70.28±12.30} & 75.83±12.98\\
DropNode &95.65±0.35 & 88.07±0.76 & { }&  65.35±2.97  & { }& 75.47±4.27 & 67.45±4.35  & { } & 73.49±4.09 & 66.00±3.16 & { } &53.96±11.34 & 61.67±10.79\\
DropEdge &95.28±0.39 & 87.54±0.60 & { }& 76.44±1.82 & { }& 72.89±5.70 & 66.27±5.31  & { } & 70.68±6.14 & 64.20±5.12& { } &67.57±5.14 & 75.00±5.10 \\
G-Mixup & 95.24±0.92 & 88.16±0.63 & { } & 62.16±2.04  & { }& 76.01±3.04 & \underline{69.41±3.21} & { } &73.68±5.67 & 65.60±4.56& { } &59.72±7.77 & 69.44±6.27 \\
C-Mixup &  \underline{96.01±0.48} &  \underline{88.40±1.44} & { } & \underline{ 59.68±1.15} & { } & 75.29±2.52 & 69.02±5.48  & { } & 74.69±4.40 & 66.40±3.36& { } &67.50±6.90 & 76.67±6.32  \\
\midrule 
\rowcolor{gray!10} \method &\textbf{96.20±0.33} & \textbf{89.44±1.06 }& { } & \textbf{56.89±1.66}  & { } & \textbf{77.01±2.59} & \textbf{69.80±3.63} & { } &  \underline{74.79±4.90} &\textbf{ 68.20±4.19 }& { } &\textbf{71.39±9.59} & \textbf{78.33±9.03} \\
\bottomrule
\end{tabular}
}
\end{table*}

\subsubsection{Implementation Details}
We equip the proposed \method with two most popular deep backbone models for biological networks, Transformer~\cite{kan2022bnt} and GIN~\cite{xu_sum}, to verify its universal effectiveness with different models. For the architecture of Transformer, the number of transformer layers is set to 2,  followed by an MLP function to make the prediction. For each transformer layer, the hidden dimension is set to be the same as the number of nodes $n$, and the number of heads is set to 4.
Regarding the GCN backbone, we set the number of GCN layers as 3. The graph representation is obtained with a sum readout function to make the final prediction.
We randomly select 70\% of the datasets for training, 10\% for validation, and the remaining for testing. In the training process, we use the Adam optimizer with an initial learning rate of $10^{-4}$ and a weight decay of $10^{-4}$. The batch size is set as 16. All the models are trained for 200 epochs, and {the epoch with the best performance on the validation set is selected for the final report}.

\subsubsection{Baselines}
We include a variety of Mixup approaches as baselines. Given $\Lambda \in [0, 1]^{ v \times v}$, $\alpha \in (0, \infty)$, $\pi \in (0, 1)$, $\cdot$ is the dot product.

\noindent\textbf{V-Mixup}~\cite{zhang2018mixup} is the vanilla Mixup by the linear combination of two random samples, 
    \begin{align}
    \begin{split}
        \tilde{S} &= (1-\lambda) S_i + \lambda S_j, 
        \tilde{y} = (1 - \lambda) y_i + \lambda y_j, \\
        \lambda &\sim \operatorname{Beta}(\alpha, \alpha).
    \end{split}
    \end{align}
\textbf{D-Mixup} is the discrete Mixup, a naive baseline designed by ourselves. Given two randomly selected samples, a synthetic sample is generated by obtaining parts of the edges from one sample and the rest from the other,
\begin{align}
\begin{split}
    \tilde{S} &= (1-\Lambda) \cdot S_i + \Lambda \cdot S_j, 
    \tilde{y} = (1 - \lambda) y_i + \lambda y_j, \\
    \Lambda^{i,j} &\sim \mathrm B(\lambda), \lambda \sim \operatorname{Beta}(\alpha, \alpha).
\end{split}
\end{align}
\textbf{DropNode}~\cite{graphsage} randomly selects nodes given a sample and sets all edge weights related to these selected nodes as zero,
\begin{align}
\begin{split}
    \tilde{S} &= \Lambda \cdot S, \Lambda^{p,:} = \Lambda^{:,p} = z, z \sim \operatorname{Bernoulli}(\pi).
\end{split}
\end{align}
\textbf{DropEdge}~\cite{Rong2020DropEdge} randomly selects edges given a sample and assigns their weights as zero,
\begin{align}
\begin{split}
    \tilde{S} &= \Lambda \cdot S, \Lambda^{p,q} \sim \operatorname{Bernoulli}(\pi).
\end{split}
\end{align}
\textbf{G-Mixup}~\cite{han2022gmixup} is originally proposed for classification tasks, which augments graphs by interpolating the generator of different classes of graphs. Since each cell in a covariance and correlation matrix represents a specific edge in a graph, we can convert a graph generator into a group of generator for each edge. We model each edge generator as a conditional multivariate normal distribution $P(S^{p,q} \mid y)$. The augmentation process can be formulated as,
\begin{equation}
\begin{aligned}
\tilde{S}^{p,q} &\sim (1 - \lambda)P(S^{p,q} \mid y_i) + \lambda P(S^{p,q} \mid y_j),
\tilde{y} = (1 - \lambda) y_i + \lambda y_j, \\
\lambda &\sim \operatorname{Beta}(\alpha, \alpha).
\end{aligned}
\end{equation}
For the setting of classification, 
\begin{equation}
    P(S^{p,q} \mid y=c) \sim \mathcal{N}\left(\mu^{p,q}_c,(\sigma^{p,q}_c)^2\right),
\end{equation}
To extend G-Mixup for regression, we slightly modify the augmentation process to adapt it for regression tasks as 
\begin{dmath}
P(S^{p,q} \mid y) \sim \mathcal{N}\left(\mu^{p,q}+\frac{\sigma^{p,q}}{\sigma_y} \rho^{p,q}\left(y-\mu_y\right), \ \left(1-(\rho^{p,q})^2\right) (\sigma^{p,q})^2\right), 
\end{dmath}
where $\mu$ and $\sigma$ are the mean and standard deviation of the weight for each edge, $\rho$ is the correlation coefficient between $S^{p,q}$ and $y$.

\noindent\textbf{C-Mixup}~\cite{yao2022cmix} shares the same process with the V-Mixup. Instead of randomly selecting two samples, C-Mixup picks samples based on label distance to ensure the mixed pairs are more likely to share similar labels $(S_j, y_j) \sim P\left(\cdot \mid\left(S_i, y_i\right)\right)$, where $P$ is a sampling function which can sample
closer pairs of examples with higher probability. For classification tasks, it degenerates into the intra-class V-Mixup.

\begin{table*}[htbp]
\centering
\small
\caption{Detailed performance comparison of different sample sizes with Transformer as the backbone.}
\vspace{-1ex}
\label{tab:percent}
\resizebox{0.8\linewidth}{!}{
\begin{tabular}{ccccccccccccccc}
\toprule
\multirow{2.5}{*}{Percentage (in \%)} &\multicolumn{4}{c}{\bf Dataset: ABCD-BioGender (AUROC$\uparrow$)} & & \multicolumn{4}{c}{\bf Dataset: ABCD-Cog (MSE$\downarrow$)}\\
\cmidrule(lr){2-5} \cmidrule(lr){7-10} 
 & {w/o Mixup} & {V-Mixup}& {C-Mixup} &\cellcolor{gray!10}{\method} & & {w/o Mixup} & {V-Mixup}& {C-Mixup} &\cellcolor{gray!10}{\method} \\
\midrule
10 &87.14±1.15 & \underline{88.99±0.75} &  88.72±1.13 & \cellcolor{gray!10}\textbf{90.21±0.64 }& { } & 73.07±2.75 & 77.00±4.58& \underline{71.22±1.68} & \cellcolor{gray!10}\textbf{70.69±1.06} \\
20 &90.60±0.91 & 91.11±0.54 &  \underline{91.49±0.89} & \cellcolor{gray!10}\textbf{92.72±0.64} & { } & 69.70±2.75 &69.80±2.42 & \underline{69.30±3.21} &\cellcolor{gray!10}\textbf{66.50±2.50} \\
30 &92.60±0.51 & \underline{93.45±0.35}  & 93.33±0.78 & \cellcolor{gray!10}\textbf{93.93±0.55} & { } & 65.97±2.48 &65.84±1.11 & \underline{64.31±0.57} & \cellcolor{gray!10}\textbf{63.50±1.61}  \\
40 &92.84±0.40 & \underline{94.06±0.48} & 93.95±0.53&\cellcolor{gray!10}\textbf{ 94.12±0.21 }& { } & 63.91±4.07 &63.14±1.08 & \underline{61.88±2.93} & \cellcolor{gray!10}\textbf{61.15±1.80 }\\
50 &94.18±0.51 & \textbf{95.20±0.39} & \underline{95.03±0.57 }& \cellcolor{gray!10}94.78±0.98 & { } & 61.89±3.85 &63.45±1.65 & \underline{61.26±1.31} & \cellcolor{gray!10}\textbf{60.82±2.71} \\
60 &94.22±0.44 & \underline{95.19±0.54} & 95.17±0.32 &\cellcolor{gray!10}\textbf{ 95.65±0.37 } & { } & \underline{59.47±1.59} & 60.32±0.94& 60.20±1.58 & \cellcolor{gray!10}\textbf{58.75±1.65} \\
70 &94.18±0.40 &\textbf{95.51±0.18} &\underline{95.49±0.28} & \cellcolor{gray!10}95.07±0.18 & { } & 62.35±2.28 & 61.15±1.51 & \underline{60.54±3.57} & \cellcolor{gray!10}\textbf{60.17±0.50} \\
80 &95.18±0.31 & 95.60±0.42 & \underline{95.73±0.51} & \cellcolor{gray!10}\textbf{95.94±0.31} & { } & \underline{59.85±1.47} & 60.31±1.07 & 60.85±3.84 & \cellcolor{gray!10}\textbf{56.78±2.05}  \\
90 & \underline{95.55±0.86} & \textbf{95.92±0.34} & {95.49±0.73} & \cellcolor{gray!10}95.24±0.65 & { } & 61.17±3.36 & 61.51±0.78 & \underline{60.35±0.93} & \cellcolor{gray!10}\textbf{57.45±3.39} \\
100 &95.28±0.32 & 95.85±0.63 & \underline{96.01±0.48} & \cellcolor{gray!10}\textbf{96.20±0.33} & { } & 60.21±1.53 & 60.43±2.67 & \underline{59.68±1.15} & \cellcolor{gray!10}\textbf{56.89±1.66}\\
\bottomrule
\end{tabular}
}
\end{table*}
\begin{figure*}
    \centering
    \includegraphics[width=1.0\linewidth]{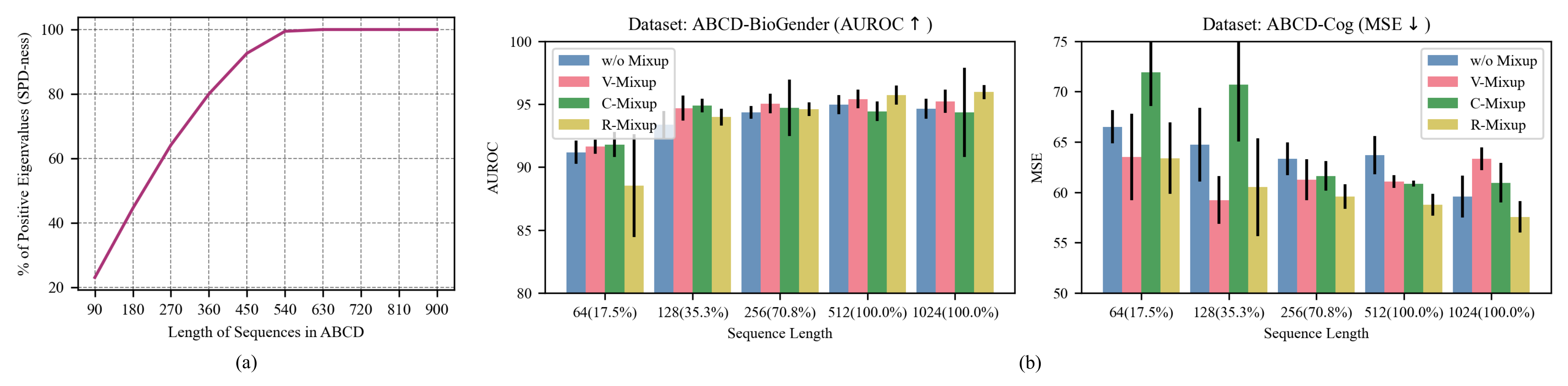}
    \vspace{-4ex}
    \caption{(a) The influence of time-series sequence length $t$ on the percentage of the positive eigenvalues (\%). (b) The influence of the sequence length $t$ or \textit{SPD-ness} (\%) on the prediction performance of classification and regression tasks.}
    \vspace{-1ex}
    \label{fig:sl}
\end{figure*}

\subsection{RQ1: Performance Comparison}\label{sec:overall}
\noindent \textbf{Overall Performance.} The overall comparison based on the Transformer and GCN backbone are presented in Table \ref{tab:overall} and Table \ref{tab:gcn} respectively, where \textit{ABCD-BioGender}, \textit{PNC}, \textit{ABIDE}, and \textit{TCGA-Cancer} focus on classification tasks, while \textit{ABCD-Cog} is a regression task. 
Since the performance of the two backbones demonstrates similar patterns, we focus on the result discussion of the Transformer due to the space limit. 
Specifically, for classification tasks, incorporating the Mixup technique can constantly improve the performance, especially on the \textit{TCGA-Cancer} dataset, which features a small sample size with high dimensional matrices. Among the various Mixup techniques, our proposed \method performs the best across datasets and tasks, indicating the further advantage of using log-Euclidean metrics instead of Euclidean metrics for SPD matrices mixture. Besides, for datasets with a relatively smaller sample size, such as PNC, ABIDE, and TCGA-Cancer, \method can further reduce training variance and stabilize the final performance compared with other data augmentation methods.

Compared with the improvements on classification tasks, \method demonstrates a more significant advantage on the regression task. It is shown that \method can significantly reduce the MSE compared with the baseline without Mixup (5.5\% with the transformer backbone) and archive a large advantage over the second runner (4.8\% with the transformer backbone).
It is also noted that other Mixup approaches sometimes hurt the model performance, indicating the Euclidean space cannot measure the distance between SPD matrices very well, and the mixed samples may not be paired with the correct labels. In contrast, our proposed log-Euclidean metric can correctly represent the distance among SPD matrices and therefore address the problem of \emph{arbitrarily incorrect label}.

\noindent \textbf{Performance with Different Sample Sizes.} 
As collecting labeled data can be extremely expensive for biological networks in practice, we adopt {\method} for the challenging low-resource setting to justify its efficacy with limited labeled data only. 
For this set of experiments, we vary the training sample size from {10\%} to {100\%} of the full datasets to show the performance of \method based on transformers with different sample sizes. 
Specifically, the ABCD dataset is adopted in this detailed analysis due to its relatively large sample size and supports for both classification and regression tasks. The selected comparing methods are the strongest baselines, namely V-Mixup and C-Mixup, from the overall performance in Table \ref{tab:overall}. Results are presented in Table \ref{tab:percent}. 

On the classification task of BioGender prediction, impressively, the proposed \method can already achieve a decent performance with only 10\% percent of full datasets and demonstrates a large margin over other compared methods. 
As the sample size becomes larger, the performance of different data augmentation methods tends to be close, while the proposed \method reaches the best performance for most of the cases (7 out of 10 setups). 
On the more challenging regression task of Cognition Summary Score prediction, \method consistently outperforms the other two baselines under different portions of the training data, which stresses the absolute advantages of our proposed \method in its flexible and effective adaption for the regression settings. Note that when equipped with an inappropriate augmentation method (i.e., V-Mixup), the regression performance can always deteriorate under different volumes of training data.  
This implies the necessity of proposing appropriate Mixup techniques tailored for biological networks 
to address specific challenges for regression tasks. Furthermore, we propose a case study in Appendix \ref{appendix:case_study} to show why R-Mixup can achieve the best performance for the Regression task in the ABCD-Cog dataset.

\subsection{RQ2: The Relations of Sequence Length, SPD-ness and Model Performance}\label{sec:spd}

To quantitatively verify the necessary conditions of SPD matrices in Proposition~\ref{prop:MCondition}, we vary the length of sequences whose pairwise correlations compose the network matrices and observe its influence on the percentage of positive eigenvalues and the final prediction performance. 
For better illustration, we define a new terminology \emph{SPD-ness} to reflect the percentage of positive values among all eigenvalues. 
The higher the percentage of positive eigenvalues, the higher \emph{SPD-ness}, and a full SPD matrice requires all the eigenvalues to be positive. Specifically, we choose the dataset with the longest time sequence, namely ABCD, to facilitate this study. Since samples in the ABCD dataset are of different sequence lengths, we simply select those with sequence length longer than 1024 and truncate them to 1024 to form a length-unified dataset \textit{ABCD-1024}, leading to 4613 samples for the \textit{ABCD-BioGender} classification task and 4533 samples for the \textit{ABCD-Cog} regression task. 

First, we investigate the relationship between the length of biological sequences $t$ and the \emph{SPD-ness} of the corresponding network matrix. The results are shown in Figure \ref{fig:sl}(a), where the value of sequence length $t$ is varied from 90 to 900 with a step size of 90. For each given $t$, we construct the correlation matrices based on each pair of the truncated sequences with only the first $t$ elements from the original sequences. Then the eigenvalue decomposition is applied to each obtained correlation matrix, and the percentage of positive ($>10^{-6}$) eigenvalues are calculated. The reported results are the average over all the correlation matrices. From this curve, we observe that the percentage of positive eigenvalues grows gradually as the time-series length increases. The growth trend gradually slows down, reaching a percentage point saturation at about the length of 540, where the full percentage indicates full \textit{SPD-ness}. Note that the number of variables $n$ for the ABCD dataset is 360. This aligns with our conclusion in Proposition~\ref{prop:MCondition} that a necessary condition for correlation matrices satisfying SPD matrices is $t \geq n$.

Second, with the verified relation between the sequence length $t$ and \textit{SPD-ness}, we study the influence of sequence length $t$ or \textit{SPD-ness} on the prediction performance. We observe that directly truncating the time series to length $t$ will lose a huge amount of task-relevant signals, resulting in a significant prediction performance drop.   
As an alternative, we reduce the original sequence to length $t$ by taking the average of each $1024/t$ consecutive sequence unit. Results on the classification task \textit{ABCD-BioGender} and the regression task \textit{ABCD-Cognition} with the input of different time-series length $t$ are demonstrated in Figure \ref{fig:sl}(b). 
It shows that for the classification task, although V-Mixup and C-Mixup demonstrate an advantage when the percentage of positive eigenvalues is low, the performance of the proposed \method continuously improves as the sequence length $t$ increases and finally beats the other baselines. 
For the regression task, our proposed \method consistently performs the best regardless of the \textit{SPD-ness} of the correlation matrices. The gain is more observed when the dataset matrices are full SPD. 
Combining these observations from both classification and regression tasks, we prove that the proposed \method demonstrates superior advantages for mixing up SPD matrices and facilitating biological network analysis that satisfies full \textit{SPD-ness}.

\begin{figure}
    \centering
    \includegraphics[width=1.0\linewidth]{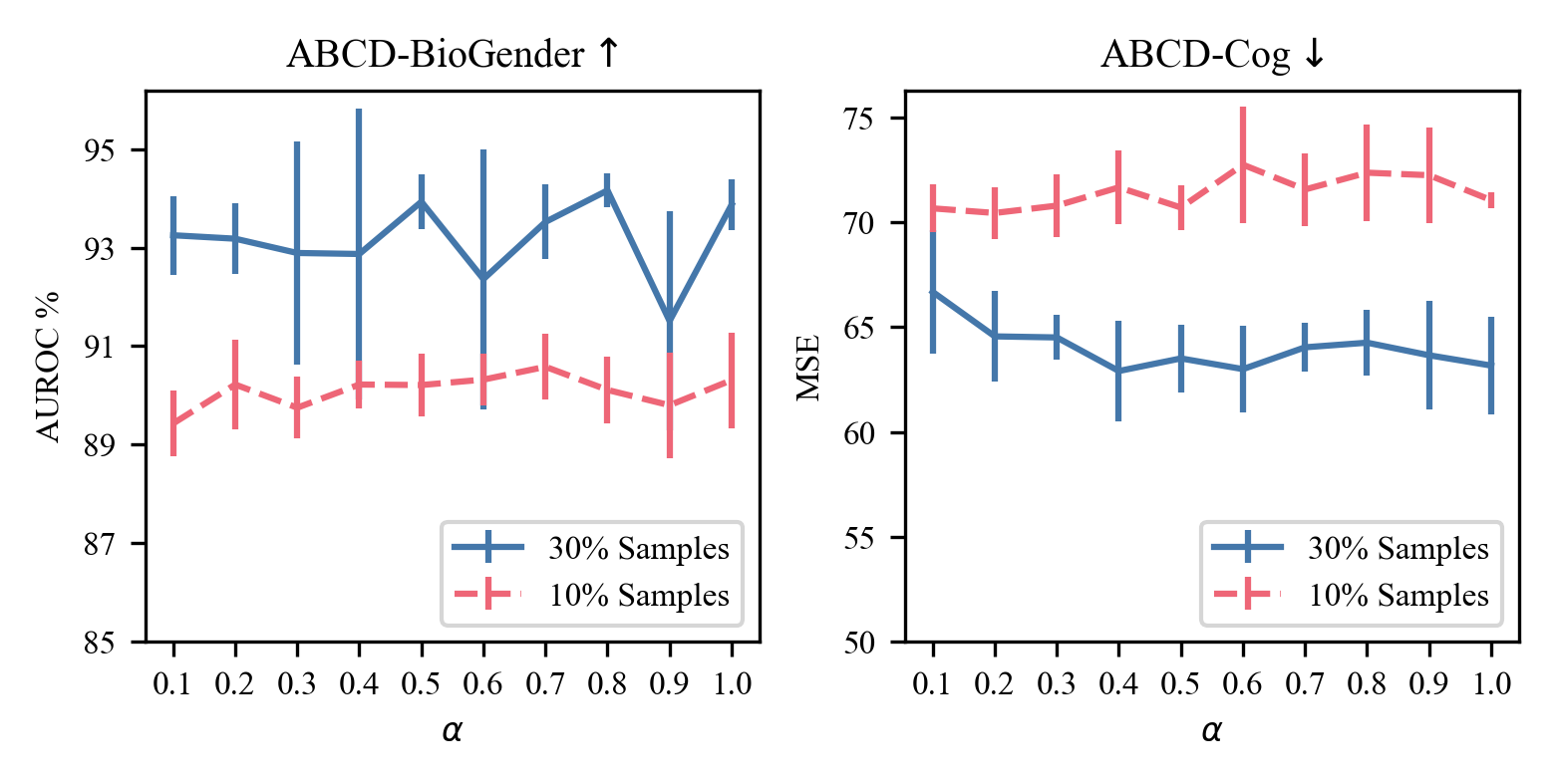}
    \vspace{-2ex}
    \caption{The influence of the key hyperparameter ($\alpha$) value on the performance of classification and regression tasks. }
    \label{fig:alpha}
    \vspace{-3.5ex}
\end{figure}

\begin{figure}
	\centering
	\includegraphics[width=0.8\linewidth]{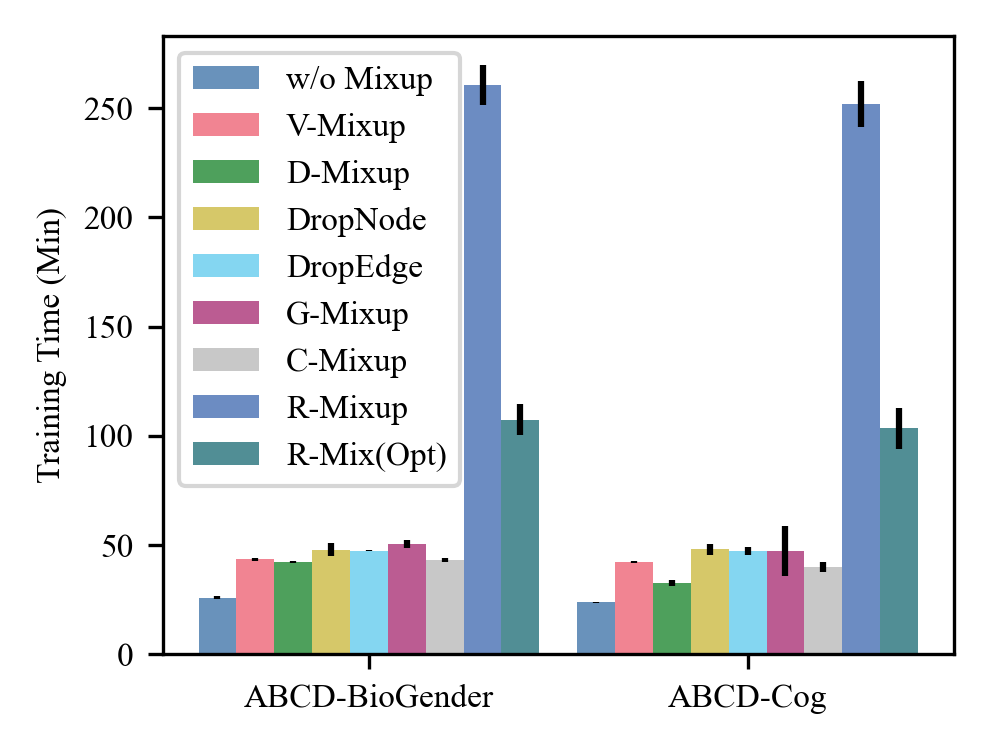}
	\caption{Training Time of different Mixup methods on the large ABCD dataset. \method is the original model while R-Mix(Opt) is time-optimized as discussed in Section \ref{sec:method_rt}.}
 \vspace{-2ex}
	\label{fig:ABCD_running}
\end{figure}

\subsection{RQ3: Hyperparameter and Efficiency Study}\label{sec:rq3}
\noindent \textbf{The Influence of Key Hyperparameter $\alpha$.} We study the influence of the key hyperparameter $\alpha$ in \method, which correspondingly changes the Beta distribution of $\lambda$ in Equation (\ref{equ: r-mixup}). Specifically, the value of $\alpha$ is adjusted from $0.1$ to $1.0$, and the corresponding prediction performance under the specific values is demonstrated in Figure \ref{fig:alpha}. We observe that the prediction performance of both classification and regression tasks are relatively stable as the value of $\alpha$ varies, indicating that the proposed \method is not sensitive to the key hyperparameter $\alpha$.

\noindent \textbf{Efficiency Study.} To further investigate the efficiency of different Mixup methods, we compare the training time of different data augmentation methods on the large-scale dataset, ABCD, to highlight the difference. The results are shown in Figure \ref{fig:ABCD_running}. Besides, the running time comparison on three smaller datasets, ABIDE, PNC, and TCGA-Cancer are also included in appendix~\ref{appendix:runningtime} for reference. All the compared methods are trained with the same backbone model~\cite{kan2022bnt}. It is observed that with the precomputed eigenvalue decomposition, the training speed of the optimized \method on the large ABCD dataset can be 2.5 times faster than the original model without optimization. Besides, on the smaller datasets such as PNC, ABIDE, and TCGA-Cancer, there is no significant difference in elapsed time between different methods.

\section{Conclusion}
\label{sec:conclusion}

In this paper, we present \method, an effective data augmentation method
tailored for biological networks that leverage the log-Euclidean distance metrics from the Riemannian manifold. We further propose an optimized strategy to improve the training efficiency of \method.
Empirical results on five real-world biological network datasets spanning both classification and regression tasks demonstrate the superior performance of \method over existing commonly used data augmentation methods under various data scales and downstream applications. Besides, we theoretically verify a necessary condition overlooked by prior works to determine whether a correlation matrix is SPD and empirically demonstrate how it affects the prediction performance, which we expect to guide future applications spreading the biological networks.
\section{Acknowledgments}
This research was supported in part by the University Research Committee of Emory University and the National Institute Of Diabetes And Digestive And Kidney Diseases of the National Institutes of Health under Award Number K25DK135913.  The authors also gratefully acknowledge support from NIH under award number R01MH105561 and R01MH118771. The content is solely the responsibility of the authors and does not necessarily represent the official views of the National Institutes of Health. Besides, the acknowledgment of used datasets can be found in Appendix \ref{sec:dataset}.
\newpage
\clearpage
\balance
\bibliographystyle{ACM-Reference-Format}
\bibliography{reference}
\clearpage

\appendix
\section{Covariance, Correlation and Positive Definite Matrices}\label{sec:SPD}

We provide detailed definitions on covariance, correlation and positive definite matrices with necessary properties here.

\begin{definition}\label{def:Cov}
	Let $X = (X_i) = (x_{ik})$ with $i = 1,...,n$ and $k = 1,...,t$ be $t$-dimensional vectors of $n$ variables. The corresponding \emph{covariance matrix} $\text{Cov}(X)$ is defined as
	\begin{align}\label{eq:Cov}
		\begin{aligned}
			\text{Cov}(X)_{ij} & = \frac{1}{t} \Big( \sum_k (x_{ik} - E(X_i)) (x_{jk} -  E(X_j)) \Big) \\
			= & E(X_i X_j) - E(X_i) E(X_j).
		\end{aligned}
	\end{align}
	The \emph{correlation matrix} is normalized as:
	\begin{align}\label{eq:Cor}
		\begin{aligned}
			\text{Cor}(X) = & \text{diag}(\frac{1}{\sqrt{ \text{Cov}(X)_{11}}},...,\frac{1}{\sqrt{\text{Cov}(X)_{vv} }}) \cdot \\
			& \text{Cov}(X) \cdot \text{diag}(\frac{1}{\sqrt{ \text{Cov}(X)_{11} }},...,\frac{1}{\sqrt{\text{Cov}(X)_{vv} }}).
		\end{aligned}
	\end{align} 
	Expressed by matrix entries, we restore the familiar \emph{Pearson correlation coefficients}:
	\begin{align}
		\text{Cor}(X)_{ij} = \frac{\text{Cov}(X)_{ij} }{\sqrt{\text{Cov}(X)_{ii}} \sqrt{\text{Cov}(X)_{jj}} }. 
	\end{align}
\end{definition}

\begin{remark}
	It should be noted that to make the definition of $\text{Cor}(X)$ valid, $\text{Cov}(X)_{ii} \neq 0$ for all $i$. Since
	\begin{align}
		\text{Cov}(X)_{ii} = E(X_i X_i) - E(X_i) E(X_i) = \frac{1}{t} \sum_k x_{ik}^2 - \Big( \frac{1}{t} \sum_k x_{ik} \Big)^2,
	\end{align}
	the geometric mean inequality says that $\text{Cov}(X)_{ii}$ vanishes only when $x_{ik}$ are identical, which does not happen in our case.  
\end{remark}

\begin{definition}
	A symmetric $n \times n$ matrix $S$ is \emph{positive semi-definite} if for any vector $u \in \mathbb{R}^n$, $u^T S u \geq 0$. Equivalently, this means that the eigenvalues of $S$ are all nonnegative. If the inequality holds strictly, $S$ is said to be \emph{positive definite}, or \emph{symmetric positive definite}, or SPD for short.
\end{definition}

\begin{proposition}\label{prop:Condition}
	Covariance and correlation matrices are positive semi-definite. A necessary condition for them to be positive definite is that the length of each sample is no less than the number of variables, i.e., $t \geq n$.
\end{proposition}
\begin{proof}
	Recall Eq.\eqref{eq:Cov} from Definition \ref{def:Cov}, let us consider column vectors $Y_k = (x_{ik} - E(X_i))$. Then $\text{Cov}(X) = \frac{1}{t} \sum_k Y_k Y_k^T$. Given any vector $u \in \mathbb{R}^n$,
	\begin{align}
		u^T \text{Cov}(X) u = \frac{1}{t} \sum_i u^T Y_k Y_k^T u =  \frac{1}{t} \sum_i (Y_k^T u )^2 \geq 0.
	\end{align} 
	On the other hand, by Eq.\eqref{eq:Cor}
	\begin{align}
		u^T \text{Cor}(X) u & = u^T \text{diag}(\frac{1}{\sqrt{ \text{Cov}(X)_{11} }},...,\frac{1}{\sqrt{\text{Cov}(X)_{vv} }}) \cdot \notag \\
		& \text{Cov}(X) \cdot \text{diag}(\frac{1}{\sqrt{ \text{Cov}(X)_{11} }},...,\frac{1}{\sqrt{\text{Cov}(X)_{vv} }}) u   \\ 
		& = \tilde{u}^T \text{Cov}(X) \tilde{u} \geq 0 \notag.
	\end{align}  
	
	If $\{Y_k\}_{k = 1}^t$ spans the whole vector space $\mathbb{R}^n$, in which case $t$ must be no less than $n$, then $\text{Cov}(X)$ is positive definite. Otherwise, there must be some vector $u$ perpendicular to all $Y_k$, which leads to $\sum_i (Y_k^T u )^2 = 0$.
\end{proof}



\section{Geodesics and Swelling Effect}\label{sec:swell}

We list the geodesic equation and Riemannian distance function induced from log-Euclidean metric on $\text{Sym}^+(n)$ here followed by a rigorous proof on the swelling effect.

\begin{definition}
	Let $\text{Sym}^+(n)$ denote the manifold of positive definite matrices equipped with the \emph{log-Euclidean metric}. Analytically, the induced distance function reads
	\begin{align}
		d(S_i, S_j) = \Vert \log S_i - \log S_j \Vert,
	\end{align}
    which measures the distance between different two points $S_i, S_j \in \text{Sym}^+(n)$, with the following geodesic connecting two points:
	\begin{align}
		\gamma(\lambda) = \exp( (1-\lambda) \log S_i + \lambda \log S_j).
	\end{align}
	Detailed derivation of the geodesics equation can be found in \cite{Sakai1996,Lee2018,Gallier2020}.
\end{definition}

\begin{proposition}[Swelling Effect]
	Given arbitrary $S_i, S_j \in \text{Sym}^+(n)$, then
	\begin{align}
		\begin{aligned}
			& \det \big( \exp( (1 - \lambda)\log S_i + \lambda \log S_j ) \big) \\
			\leq & \det \big( ( (1- \lambda)\log S_i + \lambda \log S_j ) \big).
		\end{aligned}
	\end{align} 
	Especially,
	\begin{align}
		\begin{aligned}
			& \min\{\det S_i, \det S_j\} \\
			\leq & \det \big( \exp( (1-\lambda) \log S_i + \lambda \log S_j ) \big) \leq \max\{\det S_i, \det S_j\},
		\end{aligned}
	\end{align}
	while $\det \big( (1-\lambda) \log S_i + \lambda \log S_j \big)$ would exceed the determinants of both $S_i$ and $S_j$ as shown in Figure \ref{fig:motivation2} in the main text.
\end{proposition}
\begin{proof}
	To prove the first inequality, we note one basic fact of matrix exponential: $\det (\exp(A)) = \exp (\text{Tr} A)$. Thus,
	\begin{align}
		\begin{aligned}
			& \det \big( \exp(  (1-\lambda)\log S_i + \lambda \log S_j ) \big) \\
			= & \exp \big( \text{Tr}( (1-\lambda)\log S_i + \lambda \log S_j) \big) \\
			= & \exp \big( (1-\lambda) \text{Tr} \log S_i \big) \exp\big( \lambda \text{Tr} \log S_j) \big) \\
			= & \big(\exp \text{Tr} \log S_i \big)^{1-\lambda}  \big( \exp \text{Tr} \log S_j \big)^\lambda \\
			= & (\det S_i)^{1-\lambda} (\det S_2)^\lambda.
		\end{aligned}
	\end{align}
	Then we make use of the following identity of n-dimensional Gaussian integral:
	\begin{align}
		\int \exp(-\frac{1}{2} \boldsymbol{x}^T S \boldsymbol{x}) d\boldsymbol{x} = \sqrt{\frac{(2\pi)^n}{\det S}},
	\end{align}
	where $\boldsymbol{x} \in \mathbb{R}^n$ and $S \in \text{Sym}^+(n)$. In our case,
	\begin{align}
		\begin{split}
			& \sqrt{\frac{(2\pi)^n}{ \det \big( ( (1-\lambda) S_i + \lambda S_j) \big) }} \\
			= & \int \exp(-\frac{1}{2} \boldsymbol{x}^T ( (1-\lambda) S_i + \lambda S_j) \boldsymbol{x}) d\boldsymbol{x} \\
			= & \int \Big( \exp(-\frac{1}{2} \boldsymbol{x}^T S_i \boldsymbol{x}) \Big)^{1-\lambda} \Big( \exp( -\frac{1}{2} \boldsymbol{x}^T S_j \boldsymbol{x})  \Big)^\lambda d\boldsymbol{x}  \\
			\leq & \Big( \int \exp(-\frac{1}{2} \boldsymbol{x}^T S_i \boldsymbol{x}) d\boldsymbol{x} \Big)^{1-\lambda} \Big( \int \exp( -\frac{1}{2} \boldsymbol{x}^T S_j \boldsymbol{x})  d\boldsymbol{x} \Big)^\lambda \\
			= & \sqrt{\frac{(2\pi)^n}{ (\det S_i)^{1-\lambda} (\det S_j)^\lambda }}.
		\end{split}
	\end{align}
	We use Hölder's inequality in the last step from above, which yields
	\begin{align}
		\begin{aligned}
			& \det \big( \exp( (1-\lambda)\log S_i + \lambda \log S_j) \big) \\
			= & (\det S_i)^{1-\lambda} (\det S_j)^\lambda \leq \det \big( ( (1-\lambda) S_i + \lambda S_j) \big).
		\end{aligned}
	\end{align}
	
	To prove the second inequality, let us assume $\det S_i \leq \det S_j$ and let $a = \det S_j/\det S_i \geq 1$. It is straightforward to check that $a^\lambda -1 \leq (a-1)\lambda$ when $0 \leq \lambda \leq 1$. This fact indicates that
	\begin{align}
		\begin{aligned}
			& (\frac{\det S_j}{\det S_i})^\lambda -1 \leq (\frac{\det S_j}{\det S_i} - 1)\lambda \\
			\implies & (\frac{\det S_j}{\det S_i})^\lambda \leq (\frac{\det S_j}{\det S_i} )\lambda + (1-\lambda) \\
			\implies & (\det S_i)^{1-\lambda} (\det S_j)^\lambda \leq (\det S_i)(1-\lambda) + (\det S_j)\lambda,
		\end{aligned}
	\end{align}
	and finishes the proof.
\end{proof}


\section{Kernel Regression on $\text{Sym}^+(n)$}\label{sec:kernel}

We now present the proof details of Theorem \ref{Thm:comparison} from the main text. To begin with, we introduce a method from heat kernel theory \cite{Berline2004} to generalize to Euclidean Gauss kernel
\begin{align}
	K_E(S_i, \tilde{S}) = \frac{1}{(2\pi \sigma^2)^{n^2/2}} \exp(- \frac{1}{2\sigma^2} \Vert S_i - \hat{S} \Vert^2)
\end{align}
on $\text{Sym}^+(n)$ over which our samples are distributed. The notion of \emph{geodesic regression} \cite{Peter2007,Fletcher2011} would also become apparent as we move forward. Let us first consider the following classical heat equation on Euclidean space 
\begin{align}\label{eq:HeatEquationApp}
	\Big( \frac{\partial}{\partial t} - \sum_i \frac{\partial^2}{\partial x_i^2} \Big)f = \frac{\partial f}{\partial t} + \Delta f = 0,
\end{align}
where $\Delta = \sum_i \frac{\partial^2}{\partial x_i^2}$ is the Laplacian. A solution $f(x,t)$ to this equation is interpreted as the temperature at position $x$ and time $t$. Substituting $t = \sigma^2/2$ into Eq.\eqref{eq:EGauss}, it can be check by definition that the function
\begin{align}\label{eq:EHeatKernel}
	K_t(x,y) = \frac{1}{(4\pi t)^{n^2/2}} \exp(- \frac{1}{4t} \Vert x - y \Vert^2)
\end{align}
solves Eq.\eqref{eq:HeatEquationApp}. It is called the \emph{fundamental solution} or \emph{heat kernel} as any other solutions to Eq.\eqref{eq:HeatEquationApp} can be written as the convolution with a certain function $f(y)$:
\begin{align}
	f(x,t) = \int_{\mathbb{R}^{n \times n}} K_t(x,y) f(y) dy.
\end{align} 

Based on this fact, it is natural to define Riemannian Gauss kernel as the fundamental solution to heat equation on Riemannian manifolds. To this end, we need to replace the Laplacian on Euclidean spaces by \emph{Laplace–Beltrami operator}, still denoted by $\Delta$, on manifolds. The formal definition of this operator is unnecessary here and we recommend interested readers to \cite{Berline2004,Gallier2020} for more details. It is enough to known its local coordinate expression for our purpose. Specifically, being different from Euclidean spaces with a standard and explicit \emph{coordinate system}, i.e., any vector $\boldsymbol{x} \in \mathbb{R}^n$ can be explicitly expressed by its components (coordinates) $x_i$, the coordinates of points $p$ in a manifold $M$ always need being defined exclusively depending on the concerned manifold. In the most general case, it is only known that manifolds admit local coordinate parametrizations for its local regions as they resemble Euclidean spaces. Expressed by any local coordinates,
\begin{align}\label{eq:LBOperator}
	\Delta f = \sum_{i,j} g^{ij} \Big( \frac{\partial^2 f}{\partial x_i^2} - \Gamma^k_{ij} \frac{\partial f}{\partial x_k} \Big)
\end{align}
where $g^{ij}, \Gamma^k_{ij}$ are called dual metric and Christoffel symbols and are all determined by the Riemannian metric \cite{Sakai1996,Lee2018}. Now we wish to analyze the heat equation expressed by local coordinates. However due to its intricate form involving the Riemannian metric, it is generally impossible to solve the equation analytically. Even though, the following theorem is established in heat kernel theory using advanced tools from differential geometry:

\begin{theorem} \cite{Berline2004} \label{Thm:HeatKernel}
	Let $M$ be a complete Riemannian manifold, then there exists a function $K_t(p,q)$, called heat kernel, with the following properties
	\begin{enumerate}
		\item $K_t(p,q) = K_t(q,p)$ for all $p,q \in M$.
		
		\item $\lim_{t \to 0} K_t(p,q)$ equals the Dirac delta function $\delta_x(y)$.
		
		\item $K_t(p,q)$ is positive definite and solves the heat equation.
		
		\item $K_t(p,q) = \int_M K_{t-s}(p,p') K_{s}(p',q) dp'$ for any $s > 0$.
	\end{enumerate}
\end{theorem}   

We are only interested in the third property as its confirms that the heat kernel truly determines a feature map from $\text{Sym}^+(n)$ into a higher dimensional feature space \cite{ShaweTaylor2004}. Let $\boldsymbol{y} = (y_i)$ denote the column vector consisting of training data labels, let $\boldsymbol{K}_{\tilde{S}}$ denote the column vector consisting of $K_R(S_i, \tilde{S})$ and let $G = (K_R(S_i,S_j)$ denote the kernel matrix evaluated by $K_R$ on the data set. Then the predictor function regressed through \emph{kernel ridge regression} is 
\begin{align}\label{eq:KernelRegression}
	\tilde{m}(\tilde{S}) = \boldsymbol{y}^T G^{-1} \boldsymbol{K}_{\tilde{S}}.
\end{align} 
As a reminder, if $\det G = 0$ (when does not happen here since the kernel function is positive definite), a regularization $\zeta \geq 0$ can be chosen as a trade-off between weights and square errors when optimizing the regression. The log-Euclidean metric is now explicitly used in evaluating $\boldsymbol{K}_{\tilde{S}}$ as well as $G$. On the other hand, a vanilla geodesic regression could be intuitively treat as the multi-linear regression on manifold with the Riemannian metric substituting for the Euclidean metric, which is fairly easy to deal with when we have a coordinate system and this is what we are going to do in the following proof.  

\begin{theorem}
	For $\text{Sym}^+(n)$ with log-Euclidean metric and estimators $\tilde{m}$ obtained with either geodesic regression or Gaussian kernel regression on samples. Augmented data from Riemannian geodesics bear the mean square error no more than those from straight lines.  
\end{theorem}
\begin{proof}
	We first note the fact that the exponential and logarithm functions
	\begin{align}
		\exp: \text{Sym}(n) \rightarrow \text{Sym}^+(n), \quad \log: \text{Sym}^+(n) \rightarrow \text{Sym}(n)
	\end{align}
	are \emph{isometries} between $\text{Sym}^+(n)$ and $\text{Sym}(n)$. That is: $(a)$ they are bijective and $(b)$ preserve the Riemannian distance functions:
	\begin{align}\label{eq:isometry}
		\Vert H_i - H_j \Vert = d(\exp H_i, \exp H_j), \quad d(S_i, S_i) = \Vert \log S_i - \log S_j \Vert
	\end{align} 
	for any $S_i, S_j \in \text{Sym}^+(n)$ and any $H_i, H_j \in \text{Sym}(n)$. A detailed proof on the RHS equation from can be found in \cite{Gallier2020} and with the bijectivity of $\exp$ and $\log$, we obtain the LHS equation from above. Besides, being defined as collection of all symmetric $n \times n$ matrices, $\text{Sym}(n)$ is an $\frac{1}{2}n(n-1)$-dimensional Euclidean space with a standard coordinate system introduced above. Combining with the logarithm function $\log: \text{Sym}^+(n) \rightarrow \text{Sym}(n)$, then we obtain a coordinate system for $\text{Sym}^+(n)$ within which we can express the heat equation Eq.\eqref{eq:HeatEquationApp} explicitly. 
	
	Since $\log$ is an isometry and since $\text{Sym}(n)$ is a Euclidean space, as a basic result in Riemannian geometry, the Christoffel symbols $\Gamma^k_{ij}$ in Eq.\eqref{eq:LBOperator} vanishes \cite{Lee2018} in the defined coordinate system and hence the Laplace–Beltrami operator degenerates to the common Laplacian. As a result, the fundamental solution is exactly Eq.\eqref{eq:EHeatKernel} expressed by the coordinates. After taking the inverse map $\exp$, the Euclidean distance is replaced by Riemannian distance in Eq.\eqref{eq:EHeatKernel}. Hence,
	\begin{align}\label{eq:HeatKernelApp}
		K_R(S_i, \hat{S}) = \frac{1}{(2\pi \sigma^2)^{\frac{n(n-1)}{4}}} \exp\left(- \frac{1}{2\sigma^2} d(S_i, \hat{S})^2\right),
	\end{align}
	is the heat kernel on $\text{Sym}^+(n)$ with the property being positive definite by Theorem \ref{Thm:HeatKernel}.  	
	
	Recall that the Riemannian geodesic of log-Euclidean metric is
	\begin{align}\label{eq:GeodesicApp}
		\gamma(\lambda) = \tilde{S} = \exp \left((1-\lambda) \log S_i + \lambda \log S_j \right).
	\end{align} 
	Its coordinate representations is then
	\begin{align}\label{eq:CoordinateRep}
		\log (\gamma(\lambda)) = (1-\lambda) \log S_i + \lambda \log S_j,
	\end{align}
	which is a straight line connecting $\log S_i$ and $\log S_j \in \text{Sym}(n)$. As a contrary, the coordinate representation of 
	\begin{align}\label{eq:lineApp}
		\eta(\lambda) = \tilde{S}' = (1-\lambda) S_i + \lambda S_j
	\end{align}
	is highly curved as 
	\begin{align}\label{eq:CoordinateCurve}
		\log (\eta(\lambda)) = \left( \log (1-\lambda) S_i + \lambda S_j \right).
	\end{align}
    Since $\text{Sym}(n)$ is an Euclidean space and since we verified above that the function $\log$ is an isometry, conducting geodesic regression for the samples $\{(S_i, y_i) \vert i=1,...,N\}$ is merely solving the linear model of $\{(\log S_i, y_i) \vert i=1,...,N\}$. Since the total sample number $N$ in our case is less than the dimension of the ambient Euclidean space $n^2$, the optimal solution is just a hyperplane encompassing all samples as well as those synthesized via Eq.\eqref{eq:CoordinateRep}. However, curves like Eq.\eqref{eq:CoordinateCurve} are manifestly deviated from the regression hyperplane which leads to large square loss.
	
	To verify the case involving Gaussian kernel, we make use of the following operator inequalities \cite{Carlen2009}: 
	\begin{align}\label{eq:Inequality}
		\log(  (1-\lambda) S_i + \lambda S_j ) \geq (1-\lambda) \log S_i + \lambda \log S_j.
	\end{align}
    Intuitively, the logarithm is a concave function on $(0,+\infty)$, which is generalized to hold in the setting of positive semidefinite matrices with $A \geq B$ meaning $A - B$ is positive semidefinite. For simplicity, we only analyze Eq.\eqref{eq:KernelRegression} for a pair of samples $S_i,S_j$ as an augmented sample $\tilde{S}$ is coined in this way through our mixup method. In statistics, penalty functions \cite{Yeniay2005} can be employed weaken the influence of other samples and achieve this effect. With these preparation,
    \begin{align}\label{eq:Prediction}
    	\tilde{m}(S) & = \boldsymbol{y}^T G^{-1} \boldsymbol{K}_{S} = (y_i, y_j) \begin{pmatrix} K_R(S_i,S_i) & K_R(S_i,S_j) \\ K_R(S_j,S_i) & K_R(S_i,S_i) \end{pmatrix}^{-1} \begin{pmatrix} K_R(S_i,S) \\ K_R(S_j,S) \end{pmatrix} \notag \\
    	= & \frac{1}{1 - K_{ij}^2} \Big( \big(y_i - K_{ij} y_j \big) K_{i,S} + \big(y_j - K_{ij} y_i) K_{j,S} \big) \Big),
    \end{align}
    where $K_{ij} = \exp\left(- \frac{1}{2\sigma^2} d(S_i, \hat{S})^2\right)$ is an abbreviation for the non-normalized Gaussian distribution of log-Euclidean distance with $K_{i,S}$ being denoted analogously. Substituting $\tilde{S}$ and $\tilde{S}'$ from Eq.\eqref{eq:GeodesicApp} and Eq.\eqref{eq:lineApp} into the above equation, we then compare the estimators with $\tilde{y} = (1-\lambda)y_i + \lambda y_j$ directly. 
    
    For predictions of $\tilde{S}$, we note that
    \begin{align}
    	& K_{ij} = \exp\left(- \frac{1}{2\sigma^2} \Vert \log S_i - \log S_j \Vert \right), \\
    	& K_{i,\tilde{S}} = \exp\left(- \frac{1}{2\sigma^2} \lambda \Vert \log S_i - \log S_j \Vert \right) = K_{ij}^\lambda \\
    	& K_{j,\tilde{S}} = \exp\left(- \frac{1}{2\sigma^2} (1-\lambda) \Vert \log S_i - \log S_j \Vert \right) = K_{ij}^{1 - \lambda}
    \end{align}
    with
    \begin{align}
    	\tilde{m}(\tilde{S}) = \frac{1}{1 - K_{ij}^2} \Big( K_{ij}^\lambda \big(y_i - K_{ij} y_j \big)  + K_{ij}^{1 - \lambda} \big( y_j - K_{ij} y_i) \big) \Big)
    \end{align}
    being a \emph{concave function} for $\lambda \in [0,1]$. This can be demonstrated by examining that the second order derivative 
    \begin{align}
    	\frac{d^2\tilde{m}(\tilde{S}(\lambda))}{d\lambda^2} = \frac{\ln^2 K_{ij} }{K_{ij}^\lambda (1 - K_{ij}^2) } \Big( ( K_{ij} - K_{ij}^{2\lambda + 1}) y_j + (  K_{ij}^{2\lambda} - K_{ij}^2) y_i \Big), 
    \end{align}
    which is nonnegative because $K_{ij} \leq 1$ and $K_{ij} - K_{ij}^{2\lambda+1}, K_{ij}^{2\lambda} - K_{ij}^2 \geq 0$. As a result, $\tilde{m}(\tilde{S}) \leq \tilde{y}$. On the other hand, 
    \begin{align}
    	& K_{i,\tilde{S}'} = \exp\left(- \frac{1}{2\sigma^2} \Vert \log ( (1- \lambda) S_i - \lambda S_j ) - \log S_i \Vert \right) \\
    	& K_{j,\tilde{S}'} = \exp\left(- \frac{1}{2\sigma^2} \Vert \log ( (1- \lambda) S_i - \lambda S_j ) - \log S_j \Vert \right) 
    \end{align}
    are intricate as the linear combination of matrices $( (1- \lambda) S_i - \lambda S_j )$ does not commute with the logarithm. Despite of this difficulty, we are still above to compare $\tilde{m}(\tilde{S})$ and $\tilde{m}(\tilde{S}')$ based on their general expansion in Eq.\eqref{eq:Prediction}. By \eqref{eq:Inequality},
    \begin{align}
    	\begin{aligned}
    		& \log(  (1-\lambda) S_i + \lambda S_j ) -\log S_i \geq \lambda (\log S_i + \log S_j) \\
    		\implies & \Vert \log(  (1-\lambda) S_i + \lambda S_j ) -\log S_i \Vert \geq \Vert \lambda (\log S_i + \log S_j) \Vert \\
    		\implies & K_{i,\tilde{S}'}  = \exp\left(- \frac{1}{2\sigma^2} \Vert \log ( (1- \lambda) S_i - \lambda S_j ) - \log S_i \Vert \right) \\
    		& \leq  \exp\left(- \frac{1}{2\sigma^2} \lambda \Vert \log S_i - \log S_j \Vert \right) = K_{i,\tilde{S}}. 
    	\end{aligned}
    \end{align}
    The second inequality is due to the fact that the operator norm $\Vert \text{ - } \Vert$ equals the largest eigenvalue of any positive semidefinite operator. Similar argument also implies case with $S_j$. Together with the concavity of $\tilde{m}(\tilde{S})$, Eq.\eqref{eq:Prediction} and the range of our labels, we conclude that
    \begin{align}
    	0 \leq \tilde{m}(\tilde{S}') \leq \tilde{m}(\tilde{S}) \leq \tilde{y} \implies \sum (\tilde{m}(\tilde{S}) - \tilde{y})^2 \leq (\tilde{m}(\tilde{S}') - \tilde{y})^2,
    \end{align}
    which are finally summed over the samples to show that the square error of estimation using geodesics is no more than that using straight lines on $\text{Sym}^+(n)$.
\end{proof}

\begin{remark}
	For affine-invariant metric, it has been shown that the induced \emph{Riemannian curvature tensor} $R$ is nonzero \cite{Sakai1996,Thanwerdas2023} and hence it is impossible to find coordinate systems within which $\Gamma^k_{ij} = 0$ \cite{Lee2018}. Therefore, the fundamental solution to the heat equation can never take in the concise form as Eq.\eqref{eq:HeatKernelApp} and Theorem \ref{Thm:comparison} becomes invalid to appraise the case when using affine-invariant metric.  
\end{remark}


\section{Running Time on three smaller datasets}

As shown in Figure \ref{fig:other_running}, on the smaller datasets, PNC, ABIDE, and TCGA-Cancer, there is no significant difference in elapsed time between the different methods. Notably, the proposed \method is magically faster than C-Mixup on the TCGA-Cancer dataset. This is mainly due to the small node size of TCGA-Cancer, which reduces the main barrier of the eigenvalue decomposition in \method, while the time cost of the sampling operation in the C-Mixup baseline does not change dynamically with the node size.

\label{appendix:runningtime}
\begin{figure}
	\centering
	\includegraphics[width=0.9\linewidth]{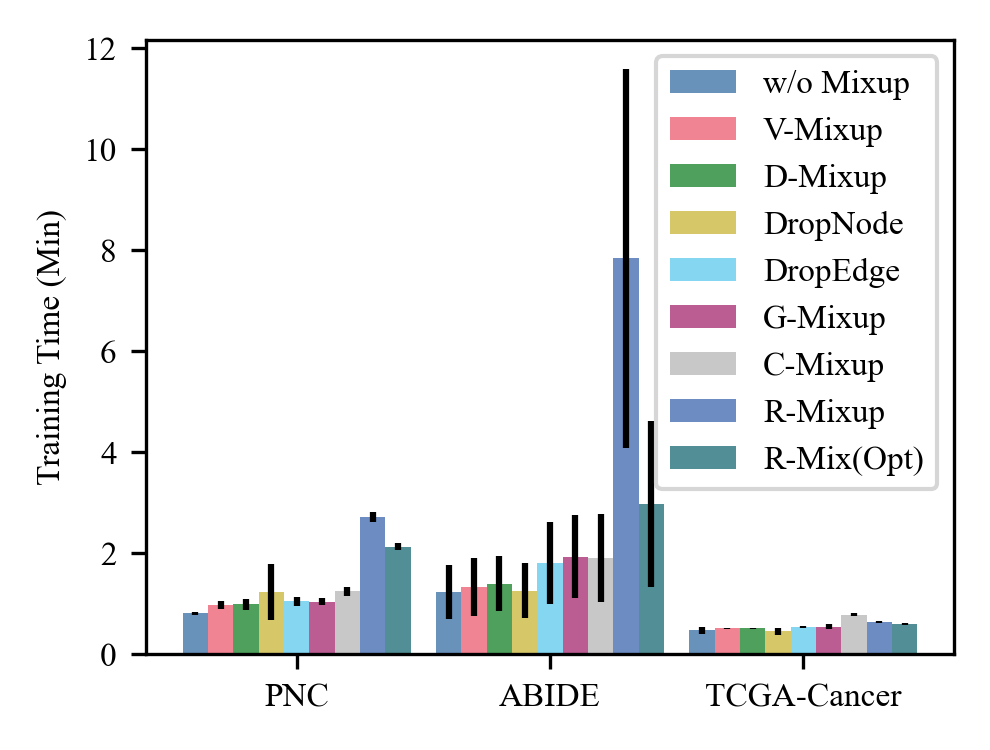}
	\caption{Running Time in PNC, ABIDE and TCGA-Cancer. R-Mixup is the original version of our method while R-Mix(Opt) is the proposed optimized version in Section \ref{sec:method_rt}.}
	\label{fig:other_running}
\end{figure}

\section{Code Implementation}
\label{appendix:code}

\begin{lstlisting}[language=Python, caption=Python Example]
import torch
import numpy as np

def tensor_log(t):
    # condition: t is symmetric.
    s, u = torch.linalg.eigh(t)
    s[s <= 0] = 1e-8
    return u @ torch.diag_embed(torch.log(s)) @ u.permute(0, 2, 1)

def tensor_exp(t):
    # condition: t is symmetric.
    s, u = torch.linalg.eigh(t)
    return u @ torch.diag_embed(torch.exp(s)) @ u.permute(0, 2, 1)
    
def r_mixup(x, y, alpha=1.0, device='cuda'):
    if alpha > 0:
        lam = np.random.beta(alpha, alpha)
    else:
        lam = 1
    batch_size = y.size()[0]
    index = torch.randperm(batch_size).to(device)
    x = tensor_log(x)
    x = lam * x + (1 - lam) * x[index, :]
    y = lam * y + (1 - lam) * y[index]
    return tensor_exp(x), y
\end{lstlisting}

\section{GCN Backbone Performance}
The performance of models with the GCN backbones can be found in Table \ref{tab:gcn}. 
\begin{table*}[htbp]
	\centering
	\small
	\caption{Overall performance comparison based on the GCN backbone. The best results are in bold, and the second best results are \underline{underlined}. The $\uparrow$ indicates a higher metric value is better and $\downarrow$ indicates a lower one is better.}
 \vspace{-1ex}
	\label{tab:gcn}
	\resizebox{1.0\linewidth}{!}{
		\begin{tabular}{ccccc cccc ccccc}
			\toprule
			\multirow{2.5}{*}{Method} &\multicolumn{2}{c}{\bf ABCD-BioGender}& & \multicolumn{1}{c}{\bf ABCD-Cog} & & \multicolumn{2}{c}{\bf PNC}& & \multicolumn{2}{c}{\bf ABIDE} & & \multicolumn{2}{c}{\bf TCGA-Cancer}\\
			\cmidrule(lr){2-3} \cmidrule(lr){5-5} \cmidrule(lr){7-8} \cmidrule(lr){10-11} \cmidrule(lr){13-14}
			& {AUROC$\uparrow$} & {Accuracy$\uparrow$}& { } & {MSE$\downarrow$}& { }  & {AUROC$\uparrow$}& {Accuracy$\uparrow$}& { } & {AUROC}$\uparrow$& {Accuracy$\uparrow$} & { }& {Precision$\uparrow$}& {Recall$\uparrow$}\\
			\midrule
			w/o Mixup & 78.82±0.62 & 71.55±0.43 & { } & 80.85±4.69 & { } & 59.14±5.66 & 60.00±4.72 & { } & 55.09±6.91 & 55.20±6.18 & { } & 30.49±6.89 & 40.83±6.18   \\
			\midrule
			V-Mixup & 81.47±0.79 & \underline{73.97±0.79} & { } & 80.35±3.09 & { } & 63.63±3.80 & \underline{61.76±3.25} & { } & 58.49±6.58 & 56.40±3.91 & { } & 38.50±9.22 & 46.67±11.18   \\
			D-Mixup & 81.30±0.35 & 73.67±0.39 & { } & 80.90±9.48 & { } & 58.68±6.24 & 58.43±5.99 & { } & 60.19±6.58 & 55.40±6.15 & { } & 37.67±5.47 & \textbf{50.00±4.17}  \\
			DropNode & 80.77±2.02 & 73.18±2.09 & { } & 88.11±10.59 & { } & \underline{63.65±5.04} & 61.57±5.16 & { } & 59.49±4.99 & 56.60±5.55 & { } & 29.58±7.69 & 39.17±10.46 \\
			DropEdge & 79.98±1.54 & 72.23±1.37 & { } & 85.98±2.31 & { } & 56.61±2.72 & 56.67±2.89 & { } & 56.58±6.78 & 54.80±4.76 & { } & \underline{39.44±7.72} & \textbf{50.00±7.80}  \\
			G-Mixup & 81.30±1.07 & 73.90±0.86 & { } & 81.28±3.46 & { } & 57.25±3.75 & 57.45±2.91 & { } & \underline{62.43±2.94} & \textbf{60.40±3.44} & { } & 38.64±8.47 & \underline{49.17±9.03}   \\
			C-Mixup & \underline{81.62±1.65} & 73.62±1.80 & { } & \underline{78.86±3.51} & { } & 60.88±7.24 & 58.24±7.61 & { } & 60.22±9.32 & 57.40±5.32 & { } & 34.17±11.74 & 46.67±15.14  \\
			\midrule
		\rowcolor{gray!10}	R-Mixup &\textbf{82.85±1.86} & \textbf{75.86±1.88} & { } & \textbf{74.88±2.03} & { } &  \textbf{64.39±5.05} &  \textbf{62.31±3.32} & { } & \textbf{63.03±5.58} & \underline{59.67±5.96} & { } & \textbf{44.78±8.64} & 48.44±8.61 \\
			\bottomrule
		\end{tabular}
	}
\end{table*}

\section{Case Study About Arbitrarily Incorrect Label Problem} 
\label{appendix:case_study}
To verify how R-Mixup migrate the \textit{arbitrarily incorrect label} problem, we design the following process:  
\begin{algorithm}
\caption{The Measurement of Arbitrarily Incorrect Label}\label{alg:cap}
\begin{algorithmic}
\State $i \gets n$
\State $d_v \gets 0$
\State $d_r \gets 0$
\While{$i > 0$}
 \State $(X_1, y_1), (X_2, y_2), (X_3, y_3) \sim \mathcal{D}_{ABCD-Cog}, \quad \text{where}  \quad y_1 < y_2 < y_3$ \Comment{Randomly sample 3 data points and sorted by $y$}
 \State $w = \frac{y_2 - y_3}{y_1-y_3}$ \Comment{Ensure $w y_1 + (1-w)y_3 = y_2$}
 \State $X_{vmix} = w X_1 + (1-w) X_3$ 
 \State $X_{rmix}= \exp \left( w \log X_1 + (1-w) \log X_3 \right)$
 \State $d_v += || X_{vmix} - X_2 ||_1$
 \State $d_r += || X_{rmix} - X_2 ||_1$
\EndWhile
\State $\overline{d_v} = \frac{d_v}{n}$
\State $\overline{d_r} = \frac{d_r}{n}$
\end{algorithmic}
\end{algorithm}

We set $n$ as 1000 and obtain $\overline{d_v}=24,416.04\text{±}4,066.60$, $\overline{d_r}=22,622.41\text{±}3,873.05$, where the sample distance $\overline{d_r}$ from R-Mixup is significantly smaller (7.3\%) than the sample distance $\overline{d_v}$ from V-Mixup. The phenomenon shows our R-Mixup indeed can migrate the \textit{arbitrarily incorrect label} problem.

\section{Dataset Acknowledgments}
\label{sec:dataset}
Part of data used in the preparation of this article were obtained from the Philadelphia Neurodevelopmental Cohort (PNC) study\footnote{\url{https://www.med.upenn.edu/bbl/philadelphianeurodevelopmentalcohort.html}} and the Adolescent Brain Cognitive Development (ABCD) Study\footnote{\url{https://abcdstudy.org}}. The PNC study support for the collection of the datasets was provided by grant RC2MH089983 awarded to Raquel Gur and RC2MH089924 awarded to Hakon Hakorson. All subjects were recruited through the Center for Applied Genomics at The Children’s Hospital in Philadelphia.  held in the NIMH Data Archive (NDA). This is a multisite, longitudinal study designed to recruit more than 10,000 children age 9-10 and follow them over 10 years into early adulthood. The ABCD Study\textsuperscript{\textregistered} is supported by the National Institutes of Health and additional federal partners under award numbers U01DA041048, U01DA050989, U01DA051016, U01DA041022, U01DA051018, U01DA051037, U01DA050987, U01DA041174, U01DA041106, U01DA041117, U01DA041028, U01DA041134, U01DA050988, U01DA051039, U01DA041156, U01DA041025, U01DA041120, U01DA051038, U01DA041148, U01DA041093, U01DA041089, U24DA041123, U24DA041147. A full list of supporters is available at  \sloppy\url{https://abcdstudy.org/federal-partners.html}. A listing of participating sites and a complete listing of the study investigators can be found at  \url{https://abcdstudy.org/consortium_members/}. ABCD consortium investigators designed and implemented the study and/or provided data but did not necessarily participate in the analysis or writing of this report. This manuscript reflects the views of the authors and may not reflect the opinions or views of the NIH or ABCD consortium investigators. The ABCD data repository grows and changes over time. The ABCD data used in this report came from NIMH Data Archive Release 4.0 (DOI 10.15154/1523041). DOIs can be found at \url{https://nda.nih.gov/abcd}.

\end{document}